\documentclass[journal]{IEEEtran}
\usepackage{amsmath,amsfonts}
\usepackage{algorithmic}
\usepackage{algorithm}
\usepackage{array}
\usepackage[caption=false,font=normalsize,labelfont=sf,textfont=sf]{subfig}
\usepackage{textcomp}
\usepackage{stfloats}
\usepackage{url}
\usepackage{verbatim}
\usepackage{graphicx}
\usepackage{cite}

\usepackage{multirow}
\usepackage{booktabs}
\usepackage{threeparttable}
\usepackage{amsthm}

\newcommand{\sd}[1]{{\scriptsize #1}}
\newtheorem{lemma}{Lemma}

\usepackage{etoolbox}
\makeatletter
\patchcmd{\@makecaption}
  {\scshape}
  {}
  {}
  {}
\makeatother

\usepackage{tikz}
\begin{document}

\title{Flexible-weighted Chamfer Distance: Enhanced Objective Function for Point Cloud Completion}

% \author{
%     \IEEEauthorblockN{Jie Li\IEEEauthorrefmark{1}, 
%                      Shengwei Tian\IEEEauthorrefmark{1}, 
%                      Long Yu\IEEEauthorrefmark{1}, and
%                      Xin Ning\IEEEauthorrefmark{2}}
%     \IEEEauthorblockA{\IEEEauthorrefmark{1}Xinjiang University, Urumqi, China}
%     \IEEEauthorblockA{\IEEEauthorrefmark{2}Institute of Semiconductors, Chinese Academy of Sciences, Beijing, China}
% }
\author{Jie Li, Shengwei Tian, Long Yu, and Xin Ning
\thanks{J. Li, S. Tian, and L. Yu are with Xinjiang University, Urumqi 830046, China. E-mail: \{lijie@stu.xju.edu.cn, tianshengwei@163.com, yul\_xju@163.com\}}
\thanks{X. Ning is with the Institute of Semiconductors, Chinese Academy of Sciences, Beijing 100083, China. E-mail: ningxin@semi.ac.cn}
\thanks{Corresponding author: Shengwei Tian (e-mail: tianshengwei@163.com).}
}

        % <-this % stops a space



\maketitle
\begin{tikzpicture}[remember picture,overlay]
    \node[anchor=south,yshift=10pt] at (current page.south) {
        \begin{minipage}{\textwidth}
            \footnotesize
            \centering
            © 2026 IEEE. Personal use of this material is permitted. Permission from IEEE must be obtained for all other uses, in any current or future media, including reprinting/republishing this material for advertising or promotional purposes, creating new collective works, for resale or redistribution to servers or lists, or reuse of any copyrighted component of this work in other works.
        \end{minipage}
    };
\end{tikzpicture}
\begin{abstract}
The Chamfer Distance (CD) is a cornerstone objective function for point cloud completion, yet its inherent symmetric weighting mechanism limits the quality of the generated results. By penalizing local detail deviations and global coverage deficiencies equally, standard CD often causes structural defects such as point aggregation and incomplete spatial structures. We introduce the Flexible-weighted Chamfer Distance (FCD), which decouples CD into local precision and global completeness sub-objectives. FCD employs an asymmetric weighting strategy that prioritizes global structural integrity, steering the optimization away from sub-optimal solutions. 
As a plug-and-play module with negligible overhead, extensive experiments on state-of-the-art networks demonstrate that FCD significantly enhances global distribution metrics while preserving local precision. Specifically, on the ShapeNet55 benchmark using AdaPoinTr, FCD reduces the Density-aware Chamfer Distance (DCD) by approximately 12.4\% (from 0.613 to 0.537), effectively mitigating point clustering. Similarly, on the PCN dataset, the proposed method reduces the Earth Mover's Distance (EMD) from 23.79 to 21.40, demonstrating superior global uniformity compared to the standard CD baseline.
Furthermore, FCD demonstrates excellent generalization. When applied to diverse tasks and datasets, including real-world scans (KITTI), industrial components (ABC), and point cloud upsampling (PU-GAN), it yields significant quantitative gains and produces visually more uniform and structurally complete point clouds. These results underscore FCD's potential as a versatile objective function for the broader point cloud generation domain.
\end{abstract}


\section{Introduction}
As a core digital representation of the 3D world, point clouds play a vital role in fields such as autonomous driving, robotics, and virtual reality. However, due to constraints like physical occlusions, sensor viewpoints, and sampling resolution, point clouds acquired in the real world are often sparse or incomplete. Point cloud completion, which aims to recover the complete geometry of an object from sparse partial observations, is a key technology for downstream applications that rely on precise 3D shape information \cite{guo2020deep}. In deep learning-driven completion methods, the choice of an objective function is paramount, as it directly guides the model's optimization path and determines the geometric fidelity and structural plausibility of the output.

In the domain of point cloud completion, Earth Mover's Distance (EMD) and Chamfer Distance (CD)~\cite{fan2017point} are widely used as both evaluation metrics and objective functions. Here, we focus on their application as objective functions. Although EMD can accurately capture differences in global distribution, its high computational complexity limits its use as a primary objective function in large-scale training. Consequently, the computationally efficient CD has become the de facto standard for point cloud completion tasks~\cite{yuan2018pcn,xie2020grnet,xiang2021snowflakenet,tang2022lake,tchapmi2019topnet,yang2017foldingnet}. However, a fundamental bottleneck underlies the widespread adoption of CD: its inherent symmetric weighting mechanism can trigger a gradient conflict between the two objectives of "local geometric fidelity" and "global structural completeness" during optimization.

Specifically, standard CD equally weights the forward term (nearest-neighbor distance from predicted points to ground-truth points), which aims to ensure local precision, and the backward term (nearest-neighbor distance from ground-truth points to predicted points), which aims to guarantee global coverage. This static balance can lead to opposing and mutually canceling gradients during optimization, trapping the network in a local minimum characterized by point clustering and difficulty in switching nearest neighbors. This results in generated point clouds with defects such as local over-densification, structural holes, or surface discontinuities. Although density-aware DCD \cite{wu2021density} offers evaluation improvements, it fails to resolve this optimization issue when used as a training objective.

This paper introduces the Flexible-weighted Chamfer Distance (FCD). Our core idea is to decouple CD into two independent sub-tasks—local precision and global coverage—and introduce an asymmetric, dynamically schedulable weighting strategy. This strategy guides the model to first prioritize the construction of a complete object topology by assigning a higher weight to the global coverage term in the early stages of training, breaking the optimization stalemate. Subsequently, the weights gradually transition to a more balanced configuration, allowing for the fine-tuning of local geometric details upon the well-established global framework. This "global-first, details-later" optimization paradigm mitigates gradient conflicts and steers the model toward convergence in a more optimal solution space. As a plug-and-play module, FCD can be seamlessly integrated into various mainstream point cloud completion networks with negligible computational cost. The main contributions of this paper are summarized as follows:
\begin{itemize}
    \item From the perspective of gradient dynamics, we identify and analyze how the symmetric weighting mechanism of standard CD is a root cause of common defects in point cloud generation, such as local clustering and structural holes.
    \item We propose the Flexible-weighted Chamfer Distance (FCD) principle, positing that an asymmetric weighting scheme (\(\beta > \alpha\)) provides a more effective and stable optimization path. We then conduct a systematic investigation into various weighting strategies (e.g., preset adaptive and uncertainty-based) that implement this principle, exploring how these schemes can be used to optimally balance the trade-off between improving global distribution metrics (e.g., EMD, DCD) and maintaining local precision (e.g., CD).
    \item Extensive experiments on representative networks (AdaPoinTr, SeedFormer) and benchmark datasets (ShapeNet55, PCN) demonstrate that FCD significantly outperforms standard CD on global distribution metrics such as DCD and EMD, while maintaining or improving local precision metrics like CD and F-Score.
    \item We validate the generalization capability of FCD across various data domains and tasks, including real-world scenarios (KITTI), complex industrial components (ABC), and point cloud upsampling, proving its broad application potential as a general-purpose objective function.
\end{itemize}

The remainder of this paper is organized as follows. Section~\ref{related work} reviews related work on point cloud analysis, completion, and objective functions. Section~\ref{method} elaborates on the proposed FCD, detailing its mathematical formulation, gradient dynamics analysis, and specific weighting strategies. Section~\ref{results} presents extensive experimental results on mainstream benchmarks (ShapeNet55, PCN) and generalization tasks (KITTI, ABC, PU-GAN), along with ablation studies and complexity analysis. Finally, Section~\ref{conclusion} concludes the paper and discusses limitations and future directions.

\section{Related Work}
\label{related work}

\subsection{Point Cloud Analysis}
Early works converted input point clouds into 2D images, e.g., MVCNN \cite{su2015mvcnn}, or 3D voxels, e.g., VoxNet \cite{maturana2015voxnet} and PVCNN \cite{liu2019pvcnn}. Following the success of PointNet \cite{qi2017pointnet} and PointNet++ \cite{qi2017pointnet++}, direct analysis of 3D coordinates became the mainstream. Subsequent works focused on enhancing local feature extraction by generalizing CNNs (PointCNN \cite{li2018pointcnn}), using density-weighted kernels (PointConv \cite{wu2019pointconv}), defining deformable kernels (KPConv \cite{thomas2019kpconv}), or constructing local graphs (PointWeb \cite{zhao2019pointweb}, DGCNN \cite{wang2019dgcnn}). More recently, Transformer-based architectures \cite{zhao2021pointtransformer, guo2021pct} and masked autoencoding techniques, such as Point-BERT \cite{yu2022pointbert}, have been introduced to learn both low-level structures and high-level semantics. 
These pioneering works have laid a solid foundation for downstream tasks such as point cloud completion.

\subsection{Point Cloud Completion}
Point cloud completion involves reconstructing a full shape from partially observed data. While methods utilizing voxels and 3D convolutions for completion \cite{dai2017shape,han2017high,stutz2018learning,liu2019point} are computationally expensive, more recent approaches have shifted towards encoder-decoder architectures based on PointNet and DGCNN to address the point cloud completion task. PCN \cite{yuan2018pcn} established a pioneering coarse-to-fine framework. Since then, researchers have explored various approaches, such as using 3D grids (GRNet \cite{xie2020grnet}), hierarchical decoders (TopNet \cite{tchapmi2019topnet}), deformation-based steps (MSN \cite{liu2020morphing}, PMP-Net \cite{wen2021pmp}), fractal-based networks (PF-Net \cite{huang2020pf}), or cascaded refinement (CRN \cite{wang2021cascaded}). With the success of Transformers, a series of completion networks have been proposed, including SnowflakeNet \cite{xiang2021snowflakenet}, PoinTr \cite{yu2021pointr}, and LAKeNet \cite{tang2022lake}. 
The baselines used in our study, SeedFormer \cite{zhou2022seedformer} and AdaPoinTr \cite{yu2023adapointr}, also leverage Transformer-based designs. Building on these foundations, recent studies have explored more complex Transformer architectures \cite{yu2024geoformer,rong2024crapcn}, generative priors like diffusion models \cite{wei2025pcdreamer}, and template-guided strategies \cite{duan2024tcorresnet}. Despite their differences in network architecture and generation paradigms, these advanced methods almost universally rely on the CD \cite{fan2017point} as their core training objective. This widespread dependence on CD makes the performance of the objective function itself a critical bottleneck, constraining the potential of these advanced network architectures. Therefore, an in-depth analysis and improvement of CD have become critically important.

\subsection{Point Cloud Objective Functions}
The computationally efficient CD is the de facto standard in point cloud completion, as the high computational cost of alternatives like EMD makes them impractical. Standard CD symmetrically aggregates two components: an accuracy term (ensuring local fidelity) and a completeness term (ensuring global coverage). However, this fixed, symmetric weighting creates an inherent conflict between the two objectives, particularly in highly expressive models. This conflict can lead to contradictory gradient updates and optimization bottlenecks. To mitigate this limitation, existing research has primarily explored three approaches. One class of methods augments the loss function, for example, by combining it with EMD \cite{wen2020point} or introducing a uniformity regularization term \cite{li2019pu}. Another class focuses on improving the distance metric itself, such as the DCD \cite{wu2021density}, which, while an excellent evaluation metric, is less effective as a training loss. The latest research attempts to fundamentally reconstruct CD, for instance, by introducing contrastive learning principles to measure distance in a learned feature space (InfoCD) \cite{lin2023infocd}, or by applying non-Euclidean geometry to better capture hierarchical structures (Hyperbolic Chamfer Distance) \cite{lin2023hyperbolic}. Although these methods seek improvements by introducing new constraints or altering the metric space, they have largely overlooked the fundamental conflict caused by CD's symmetric weighting from the perspective of the optimization process and gradient dynamics. Our work fills this gap by being the first to focus on guiding the optimization path through a flexible gradient weighting scheme.

\subsection{Multi-task Learning}
Multitask learning aims to leverage shared representation learning to concurrently address multiple tasks, achieving broad success across natural language processing \cite{liu2015representation,liu2019multi}, speech processing \cite{hu2015fusion,wu2015deep}, and computer vision \cite{leang2020dynamic}. In this learning framework, the losses from various tasks are amalgamated through a weighted approach, with weight-setting strategies including static weighting 
\cite{qu2019attentive,yeh2019flowdelta} and dynamic weighting \cite{leang2020dynamic,belharbi2016deep,chen2018gradnorm,liu2019loss,ning2021uncertainty}. Considering the two components of CD as distinct learning objectives, we introduce FCD to instruct the training of point cloud completion networks. FCD seeks to employ diverse weighting strategies to alleviate the constraints of CD, thus elevating the quality of outcomes produced.

\begin{table}[t]
    \centering
    \caption{Table of Notations.}
    \label{tab:notations}
    \resizebox{\linewidth}{!}{%
    \begin{tabular}{@{}lll@{}}
        \toprule
        \textbf{Symbol} & \textbf{Description} \\
        \midrule
        \(\mathcal{P}, \mathcal{G}\) & Predicted point cloud and ground-truth point cloud, respectively. \\
        \(\mathcal{P}_{\text{in}}\) & The input incomplete point cloud to the network. \\
        \(\mathcal{P}_{\text{coarse}}, \mathcal{P}_{\text{fine}}\) & The coarse and fine point clouds generated by the network, respectively. \\
        \(p, g\) & An individual point from \(\mathcal{P}\) and \(\mathcal{G}\), respectively.\\
        \(|\cdot|\) & The cardinality of a point set (i.e., the number of points). \\
        \(\|\cdot\|_2\) & The L2 norm, representing Euclidean distance. \\
        \(d_{\text{CD}}\) & The standard Chamfer Distance. \\
        \(d_{\text{CD}_{\text{local}}}\) & The local-fitting term of CD (from predicted to ground-truth). \\
        \(d_{\text{CD}_{\text{global}}}\) & The global-coverage term of CD (from ground-truth to predicted). \\
        \textbf{\(d_{\text{FCD}}\)} & \textbf{Our proposed Flexible-weighted Chamfer Distance.} \\
        \(\alpha, \beta\) & The weights for the local-fitting and global-coverage terms in FCD. \\
        \(\theta, \tau\) & The upper and lower bounds for the FCD weights in our weighting strategies. \\
        \(\epsilon, T\) & The current and total training steps, respectively. \\
        \(t\) & The transition epoch for dynamic weighting (e.g., Stair, Abridged Linear). \\
        \(\mathcal{L}_{\text{coarse}}, \mathcal{L}_{\text{fine}}\) & The loss calculated for the coarse and fine stages, respectively. \\
        \(\mathcal{L}_{\text{total}}\) & The total training loss, combining coarse and fine stage losses. \\
        K & The total number of coarse stages. \\
        \(\mathcal{P}_c^{(k)}\) & The \(k\)-th coarse point cloud from the intermediate stages. \\
        \(\mathcal{W}(\epsilon)\) & The weighting scheduler function at training step \(\epsilon\). \\
        \bottomrule
    \end{tabular}}

\end{table}

\begin{figure}[t]
  \centering
  \includegraphics[width=\linewidth]{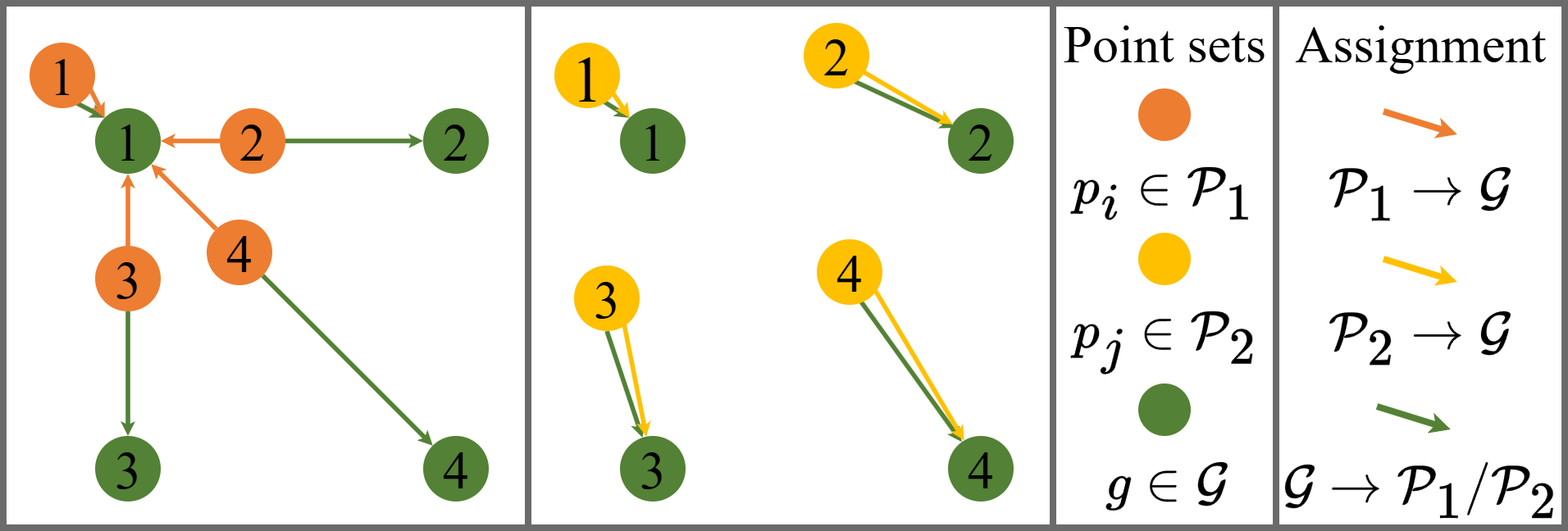}
  \caption{The ambiguity of standard CD. Two predicted point clouds \(\mathcal{P}_1\) (clustering) and \(\mathcal{P}_2\) (uniform) yield the same CD value relative to the ground truth \(\mathcal{G}\).}
  \label{fig:explain-cd}
\end{figure}

\section{Flexible-weighted CD as an Objective Function}
\label{method}
To address the symmetric weighting limitation of the standard CD, this section details our proposed FCD. We begin by reviewing the fundamental definitions of CD and related metrics. Subsequently, we analyze the mathematical formulation of FCD, its advantages from a gradient dynamics perspective, and its specific weighting strategies. Finally, we demonstrate its integration as an objective function into end-to-end networks. For clarity, Table~\ref{tab:notations} summarizes the main notations used in this paper and their definitions.

\subsection{Preliminaries}
\label{prelimi}

Assuming two point sets, the predicted point cloud \(\mathcal{P}\) and the ground-truth point cloud \(\mathcal{G}\), the Chamfer Distance (CD) between them can be defined as:
\begin{equation}
    d_{\text{CD}}(\mathcal{P}, \mathcal{G}) = \frac{1}{|\mathcal{P}|} \sum_{p \in \mathcal{P}} \min_{g \in \mathcal{G}} \|p - g\|_2 + \frac{1}{|\mathcal{G}|} \sum_{g \in \mathcal{G}} \min_{p \in \mathcal{P}} \|g - p\|_2,
\end{equation}
where \(|\mathcal{P}|\) and \(|\mathcal{G}|\) represent the number of points in \(\mathcal{P}\) and \(\mathcal{G}\), respectively, and \(\|\cdot\|_2\) denotes the L2 norm, i.e., the Euclidean distance. In practice, there are two common variants of CD, depending on whether the distance term is squared. Following convention \cite{yuan2018pcn,zhou2022seedformer,yu2023adapointr}, we define them as follows:
\begin{equation}
    % \begin{split}
        d_{\text{CD-}\ell_1}(\mathcal{P}, \mathcal{G}) = \frac{1}{2} \big( \frac{1}{|\mathcal{P}|} \sum_{p \in \mathcal{P}} \min_{g \in \mathcal{G}} \|p - g\|_2 
         + \frac{1}{|\mathcal{G}|} \sum_{g \in \mathcal{G}} \min_{p \in \mathcal{P}} \|g - p\|_2 \big),
    % \end{split}
\end{equation}
\begin{equation}
    d_{\text{CD-}\ell_2}(\mathcal{P}, \mathcal{G}) = \frac{1}{|\mathcal{P}|} \sum_{p \in \mathcal{P}} \min_{g \in \mathcal{G}} \|p - g\|_2^2 
    + \frac{1}{|\mathcal{G}|} \sum_{g \in \mathcal{G}} \min_{p \in \mathcal{P}} \|g - p\|_2^2,
\end{equation}
where \(\text{CD-}\ell_1\) (the \(\ell_1\) version) uses the direct Euclidean distance and offers better robustness to outliers. \(\text{CD-}\ell_2\) (the \(\ell_2\) version) uses the squared Euclidean distance, which is more sensitive to subtle changes in shape by amplifying distance differences but is correspondingly more susceptible to outliers. Unless otherwise specified, the FCD method proposed hereafter is applicable to both variants. The core of CD is to compute the distance from each point \(p \in \mathcal{P}\) (or \(g \in \mathcal{G}\)) to its nearest point in \(\mathcal{G}\) (or \(\mathcal{P}\)) and sum these distances to obtain an overall similarity measure between the two point sets. It does not require a one-to-one correspondence between points, and its simple and flexible form allows for high computational efficiency and good generalization. In contrast, the Earth Mover's Distance (EMD) is defined as:
\begin{equation}
    d_{\text{EMD}}(\mathcal{P}, \mathcal{G}) = \min_{\varphi: \mathcal{P} \rightarrow \mathcal{G}} \sum_{p \in \mathcal{P}} \|p - \varphi(p)\|_2.
\end{equation}
EMD calculates the minimum distance required to move one set of points to another, where \(\mathcal{P}\) and \(\mathcal{G}\) must be point sets of equal size, and \(\varphi:\mathcal{P}\rightarrow \mathcal{G}\) denotes a mapping function that maps each point \(p\) in \(\mathcal{P}\) to a point 
\(\varphi(p)\) in \(\mathcal{G}\). Employing EMD for supervised point cloud completion can surpass CD in effectiveness but often at a high computational cost. The newly introduced Density-aware Chamfer Distance (DCD) \cite{wu2021density} is an innovative approach for assessing point set similarity, specifically aimed at evaluating point cloud completion results, which is defined as: 
\begin{equation}
    \begin{split}
         d_{\text{DCD}}(\mathcal{P}, \mathcal{G}) = \frac{1}{2} \big( \frac{1}{|\mathcal{P}|} \sum_{p \in \mathcal{P}} (1 - \frac{1}{n_{\hat{g}}} e^{-\alpha \|p-\hat{g}\|_2}) \\
         + \frac{1}{|\mathcal{G}|} \sum_{g \in \mathcal{G}} (1 - \frac{1}{n_{\hat{p}}} e^{-\alpha \|g-\hat{p}\|_2}) \big),
    \end{split}
\end{equation}
where \(\hat{p}=\min_{p\in\mathcal{P}}\|g-p\|_2\), \(\hat{g}=\min_{g\in\mathcal{G}} \|p-g\|_2\), and \(n_{\hat{p}}\) denotes the count of points in \(\mathcal{G}\) closest to \(\hat{p}\in\mathcal{P}\) , and vice versa. The fundamental principle of DCD involves introducing query frequency to account for the local density distribution of points, thereby enhancing the sensitivity of distance calculations to local density. Additionally, DCD's value range is bounded, typically within [0, 1], preventing excessive sensitivity to outliers exhibiting quadratic growth. This ensures its stability and rationality in assessing point cloud completion outcomes. The scaling factor \(\alpha\) adjusts sensitivity, commonly set as \(\alpha=1000\) when serving as an evaluation metric.

\subsection{Formulation and Interpretation}
\label{para:fcd-explain}

\begin{figure}[t]
  \centering
  \includegraphics[width=\linewidth]{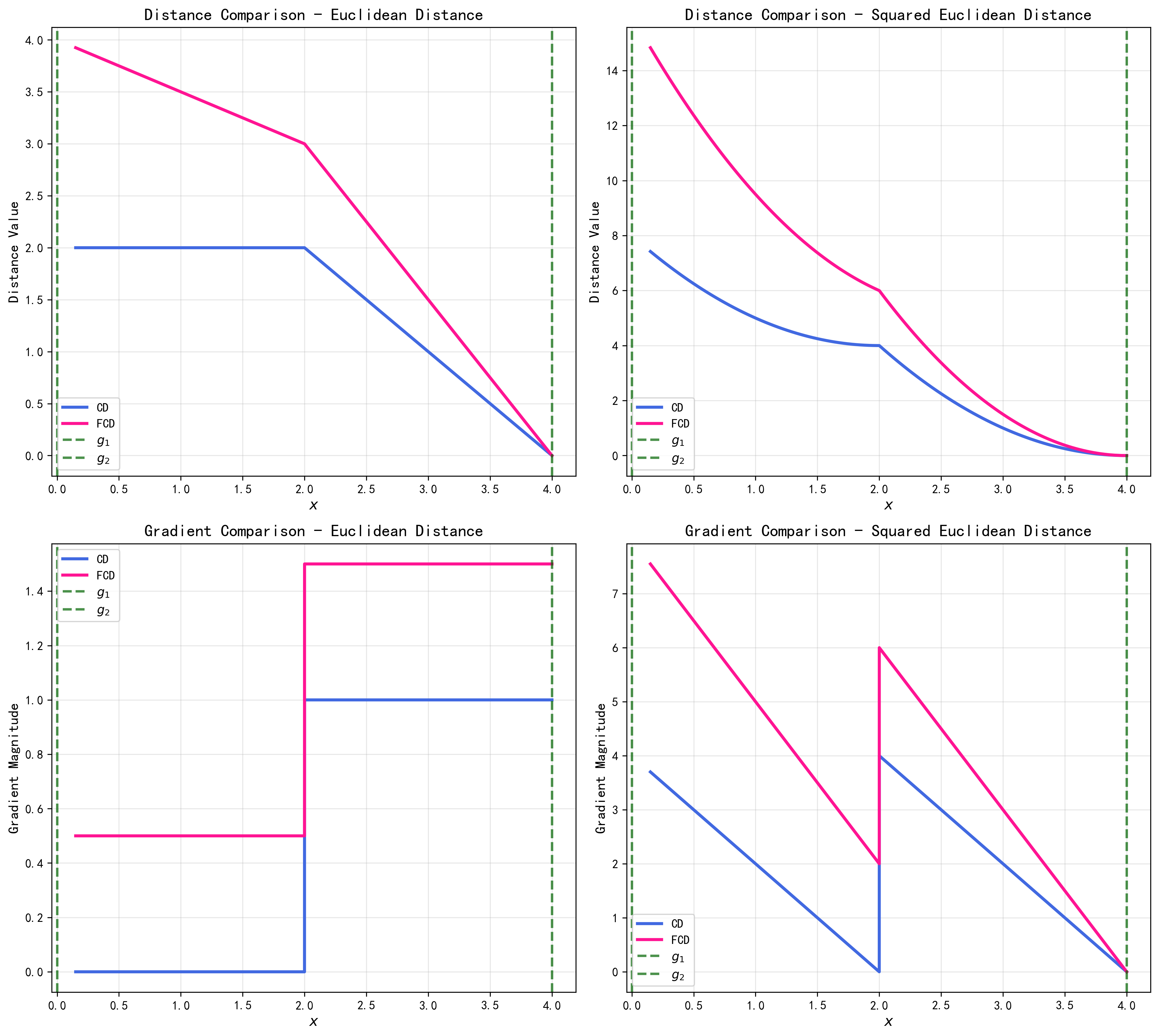}
  \caption{Gradient and distance analysis for CD vs. FCD. The values and gradient changes of CD and FCD for \(p_2=[x,0]\) at different positions on the line connecting \(g_1=[0,0]\) and \(g_2=[4,0]\) (while ensuring the distance from \(p_2\) to \(g_1\) is greater than from \(p_1\) to \(g_1\)). FCD maintains a non-vanishing gradient, while the standard CD gradient vanishes, causing an optimization stalemate.}
  \label{fig:comparison}
\end{figure}

CD is a classic metric for measuring the similarity between two point clouds—typically a predicted and a ground-truth point cloud in generation scenarios—and is often used directly as an objective function. For analytical purposes, we decompose it into its two directional components:

\begin{equation}
    d_{\text{CD}_\text{{local}}} = \frac{1}{|\mathcal{P}|} \sum_{p \in \mathcal{P}} \min_{g \in \mathcal{G}} \|p - g\|_2,
\end{equation}
\begin{equation}
    d_{\text{CD}_\text{{global}}} = \frac{1}{|\mathcal{G}|} \sum_{g \in \mathcal{G}} \min_{p \in \mathcal{P}} \|g - p\|_2,
\end{equation}

where \(||\cdot||_2\) denotes the L2 norm (Euclidean distance), and as is common in CD-based methods, the squared L2 norm (squared Euclidean distance) can also be used in practice. The first term, \(d_{\text{CD}_\text{{local}}}\) (local-fitting term), originates from the predicted points and measures the nearest distance from each predicted point to the ground-truth point cloud. It only requires each predicted point to be "close" to some ground-truth point and does not enforce coverage of all ground-truth points. Thus, it reflects the local precision and detail-fitting capability of the predicted point cloud. The second term, \(d_{\text{CD}_\text{{global}}}\) (global-coverage term), originates from the ground-truth points and requires that each ground-truth point can find a nearest counterpart in the predicted point cloud, reflecting the completeness of the predicted point cloud's global coverage of the true structure. Owing to its ability to balance overall structural consistency with local detail precision, CD has been widely adopted as an evaluation metric and objective function; however, it still exhibits inherent limitations.

\textbf{As an evaluation metric.}
A single scalar value may correspond to markedly different local/global trade-off states, such as excellent local fitting with incomplete coverage, or vice versa. This trade-off ambiguity can result in predicted point clouds with disparate distributional morphologies yielding identical CD values. For intuitive illustration, Fig.~\ref{fig:explain-cd} depicts two predicted results \(\mathcal{P}_1\) and \(\mathcal{P}_2\) relative to \(\mathcal{G}\), both exhibiting the same CD but with \(\mathcal{P}_1\) displaying evident local over-clustering and \(\mathcal{P}_2\) demonstrating a more balanced global structural distribution.

\textbf{As an objective function.}
Traditional CD is equivalent to assigning symmetric weights (\(\alpha=\beta\)) to the two terms, making the gradient contributions from the local and global terms equal in their scalar weighting. While formally reasonable, in practice, when predicted points cluster (as in the case of \(\mathcal{P}_1\) , often caused by network capacity or initialization), the angle between the local and global term gradients for some points may be obtuse. This creates a "mutual cancellation" effect, making it easy to get trapped in a local optimum. To alleviate this problem, we propose the Flexible-weighted Chamfer Distance (FCD):
\begin{equation}
    d_{\text{FCD}}(\mathcal{P}, \mathcal{G}) = {\alpha}d_{\text{CD}_\text{{local}}} + {\beta}d_{\text{CD}_\text{{global}}}, \qquad \alpha, \beta>0, 
\end{equation}
where \(\alpha\) and \(\beta\) regulate local fitting precision and global coverage completeness, respectively. By adopting an asymmetric design with \(\beta > \alpha\) in the early stages of training, we can proactively break the state where predicted points cluster and fail to "disperse for coverage". Evidently, traditional CD is a special case where \(\alpha=\beta\) , making FCD backward-compatible with the vast number of existing CD-based methods. In practice, either Euclidean distance (\(\ell_1\) version) or squared Euclidean distance (\(\ell_2\) version) can be used.

To analyze the advantages of FCD from the perspective of objective function optimization and gradient dynamics, we first present the following lemma.

\begin{lemma}

For \(p,g \in \mathbb{R}^3\) and \(p \neq g\), let the Euclidean distance be \(d(p,g)=\|p-g\|_2\) and the squared Euclidean distance be \(d^2(p,g)=\|p-g\|_2^2\). Then,
\begin{equation}
\frac{\partial d(p,g)}{\partial p} = \frac{p-g}{\|p-g\|_2}, \qquad \frac{\partial d^2(p,g)}{\partial p} = 2(p-g).
\end{equation}
\end{lemma}

\begin{proof}
    Let \(d(p,g)=\sqrt{(p\!-\!g)^T(p\!-\!g)}=\sqrt{\sum_{i=1}^3 (p_i \!-\! g_i)^2}\). By the chain rule,
    \begin{equation}
    \frac{\partial d(p,g)}{\partial p} = \frac{1}{2\sqrt{\sum_{i=1}^3(p_i-g_i)^2}}\cdot 2(p-g)=\frac{p-g}{\|p-g\|_2}.
    \end{equation}

    And for \(d^2(p,g)=(p-g)^T(p-g)=\sum_{i=1}^3 (p_i-g_i)^2\), direct differentiation gives
    \begin{equation}
    \frac{\partial d^2(p,g)}{\partial p}=2(p-g).
    \end{equation}
\end{proof}

To simplify the analysis, we only examine the matching and gradient behavior of the subsets \(\mathcal{P}_1'=\{p_1,p_2\}\) and \(\mathcal{G}'=\{g_1,g_2\}\) as shown in Fig.~\ref{fig:explain-cd}, and we set the FCD weights to \(\alpha=1, \beta=2\). Here, \(p_1\) and \(g_1\) have already formed a good match (its force vector is stable and points towards the optimum), so we focus on the movement of \(p_2\). When \(p_2\) is located on the line segment between \(g_1\) and \(g_2\), and the distance from \(p_2\) to \(g_1\) is greater than the distance from \(p_1\) to \(g_1\), the global term expects it to "disperse" towards \(g_2\). However, before it crosses the midpoint (after which its nearest neighbor for the local term would switch from \(g_1\) to \(g_2\)), the incorrect pairing of the local term still "pulls" it back towards \(g_1\), hindering the assignment switch.

Before crossing the midpoint, we have (using \(d^{r}\) to denote either the first-order distance \(r=1\) or squared distance \(r=2\)):
\begin{equation}
\frac{\partial d_{\text{CD}}(\mathcal{P}_1',\mathcal{G}')}{\partial p_2} = \frac{1}{2}\frac{\partial d^{r}(p_2,g_1)}{\partial p_2}+\frac{1}{2}\frac{\partial d^{r}(g_2,p_2)}{\partial p_2},
\end{equation}

\begin{equation}
\frac{\partial d_{\text{FCD}}(\mathcal{P}_1',\mathcal{G}')}{\partial p_2} = \frac{1}{2}\frac{\partial d^{r}(p_2,g_1)}{\partial p_2}+\frac{\partial d^{r}(g_2,p_2)}{\partial p_2}.
\end{equation}
When \(r=1\) (using Euclidean distance), the two gradient terms have opposite directions and equal coefficients, resulting in
\begin{equation}
\frac{\partial d_{\text{CD}}}{\partial p_2}=0, \qquad \frac{\partial d_{\text{FCD}}}{\partial p_2} = \frac{1}{2}\frac{p_2-g_2}{\|p_2-g_2\|_2}, 
\end{equation}
which means CD provides no driving force for \(p_2\) at this stage, whereas FCD still maintains an effective gradient. When \(r=2\) (using squared Euclidean distance):
\begin{equation}
\frac{\partial d_{\text{CD}}}{\partial p_2}=2p_2-(g_1+g_2), \qquad \frac{\partial d_{\text{FCD}}}{\partial p_2}=3p_2-(g_1+2g_2),
\end{equation}
FCD exerts a stronger "pull" in the global direction (towards \(g_2\)).

To further illustrate this, we show in Fig.~\ref{fig:comparison} the function values and gradient changes of CD and FCD for \(p_2=[x,0]\) at different positions on the line connecting \(g_1=[0,0]\) and \(g_2=[4,0]\) (while ensuring the distance from \(p_2\) to \(g_1\) is greater than from \(p_1\) to \(g_1\)). We observe that: (i) When using CD (especially the \(\ell_1\) version), the effective gradient for \(p_2\) tends to vanish (or decay rapidly) before it crosses the midpoint \([2,0]\), making it difficult to complete the nearest-neighbor switch; (ii) When using FCD with \(\beta>\alpha\), the additional gradient provided by the global term breaks the "stalemate," enabling \(p_2\) to cross the midpoint and continue evolving towards a globally balanced distribution.

\textbf{Takeaway.} Emphasizing the global term (\(\beta > \alpha\)) supplies a non-vanishing global gradient that enables assignment switches by breaking optimization stalemates, mitigating stagnation and promoting uniform coverage while retaining CD’s geometric interpretability.

\begin{figure}[t]
  \centering
  \includegraphics[width=0.9\linewidth]{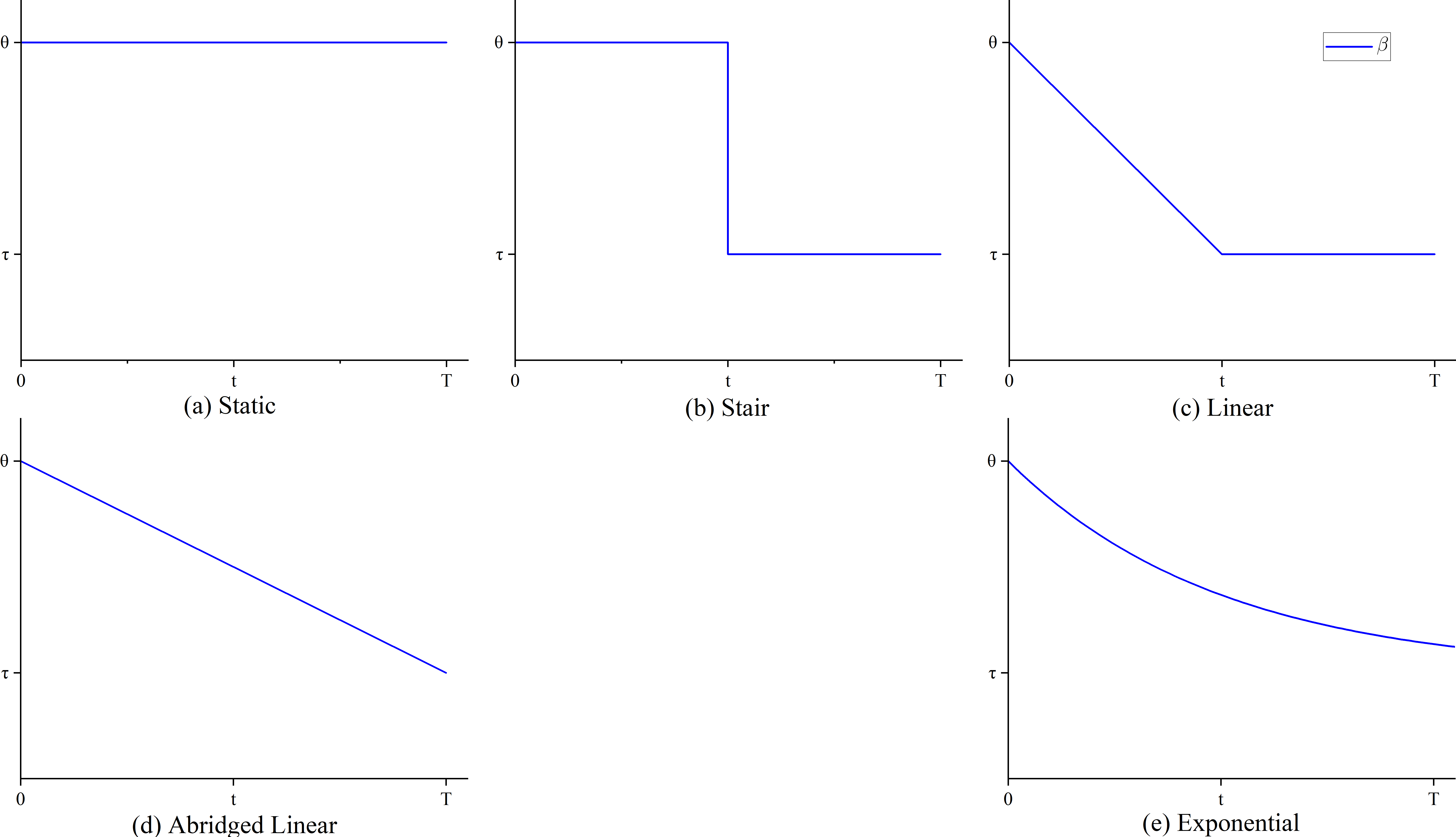}
  \caption{The five Preset Adaptive Weighting schedules for the global coverage weight (\(\beta\)) over training steps (\(T\)).}
  \label{fig:weighting}
\end{figure}

\subsection{Weighting Strategy}
\label{para:weighting-strategy}

The FCD principle of \(\beta > \alpha\) (Section~\ref{para:fcd-explain}) provides a general design framework. A key practical question is how to apply this principle to best manage the inherent trade-off between global coverage and local precision. Therefore, we now introduce and analyze several complementary families of weighting strategies. This investigation aims to understand how different scheduling approaches impact the optimization path and the final balance between global (DCD, EMD) and local (CD, F-Score) performance metrics. (i) Preset Adaptive Weighting, which follows human-designed schedules to gradually shift emphasis; and (ii) Uncertainty Weighting, which adapts weights automatically based on homoscedastic task uncertainty. We further initialize the uncertainty-based scheme with \((\beta,\alpha) = (\theta, \tau)\) to bias the optimization toward global uniformity at the start, after which the weights evolve data-dependently.

\begin{figure*}[t]
    \centering
    \includegraphics[width=\linewidth]{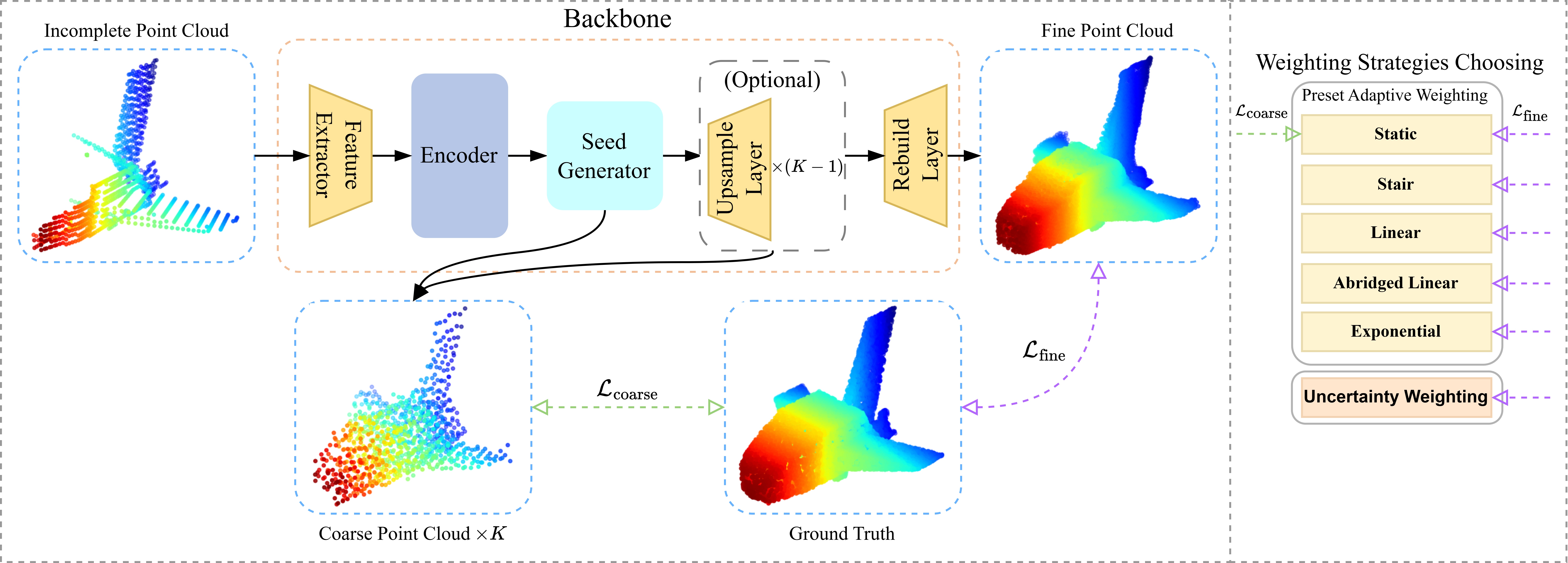}
    \caption{Architecture of the point cloud completion network integrating Flexible-weighted Chamfer Distance (FCD). The network follows a coarse-to-fine paradigm, producing both coarse and fine outputs. FCD is applied at multiple stages (see Algorithm~\ref{alg:fcd_training} for training schedule): static weighting is used at the coarse stage to establish global structure, while various adaptive weighting schemes are employed at the fine stage to balance local detail refinement and global coverage.
    }
    \label{fig:frame}
\end{figure*}

\subsubsection{Preset Adaptive Weighting}
This family of strategies investigates how predefined schedules for FCD's weights impact the performance trade-off. The primary objective is to secure strong global performance, while simultaneously investigating if these relative weight changes can maintain or even improve local performance. In all preset methods, we set an upper limit \(\theta\) and a lower limit \(\tau\). The local objective weight \(\alpha\) is held constant at the lower limit for all schedules, i.e., \(\alpha = \tau\) for all epochs. The global objective weight \(\beta\) is then manipulated. With the exception of the static baseline, these schedules are designed to gradually evolve the global weight \(\beta\) from the upper limit \(\theta\) towards the lower limit \(\tau\). We propose five distinct schedules, illustrated in Fig.~\ref{fig:weighting}, to manage the evolution of $\beta$ during training. The five schedules are described as follows:

\paragraph{Static schedule}
Treat the global objective as primary and the local objective as secondary: set \(\alpha=\tau\) and \(\beta=\theta\) for all epochs.

\paragraph{Stair schedule}
Emphasize global structure early (\(\beta=\theta\)), then switch to \(\beta=\tau\) at epoch \(\epsilon=t\).

\paragraph{Linear schedule}
This schedule linearly reduces the global weight from \(\theta\) to \(\tau\) over the total training duration of \(T\) epochs. The weight at epoch \(\epsilon\) is \(\beta_{\epsilon} = \theta - \frac{\epsilon}{T}(\theta - \tau)\).

\paragraph{Abridged Linear schedule}
This schedule holds \(\beta=\theta\) until epoch \(\epsilon=t\), and then begins a linear decay, reaching \(\tau\) at the final epoch \(T\).

\paragraph{Exponential schedule}
This schedule exponentially decays \(\beta\) from \(\theta\) towards \(\tau\). The weight at epoch \(\epsilon\) is given by \(\beta_{\epsilon} = (\theta - \tau) e^{-\frac{\epsilon}{\sigma}} + \tau\), where \(\sigma\) controls the decay rate.

\subsubsection{Uncertainty Weighting}

Kendall \textit{et al.} \cite{kendall2018multi} proposed treating multi-task learning as tasks with equal variances, where task-related weights are automatically adjusted based on task uncertainty. This approach derives the multi-task objective function by maximizing the Gaussian likelihood function. The task weights are adjusted automatically during training based on changes in the loss values and are negatively correlated with the variance of tasks, meaning that tasks with higher uncertainty have smaller weights. However, in this method, the importance of different tasks is considered equal. To make the network focus more on the global distribution, we slightly modified the uncertainty weighting approach. Specifically, we assigned initial weights of \(\theta\) and \(\tau\) to the global and local objectives, respectively (consistent with the Preset Adaptive Weighting method), before allowing them to adjust automatically based on homoscedastic uncertainty during training.

\subsection{Application as an Objective Function}

\begin{algorithm}[t]
\caption{End-to-end Training with Flexible-weighted Chamfer Distance (FCD)}
\label{alg:fcd_training}
\begin{algorithmic}[1]
\REQUIRE Incomplete cloud \(\mathcal{P}_{\text{in}}\), ground truth \(\mathcal{G}\); stages \(K\); epochs \(T\); distance order \(r\!\in\!\{1,2\}\) (\(\ell_1\)/\(\ell_2\)); preset bounds \((\tau,\theta)\) with \(\theta>\tau\); weighting scheduler \(\mathcal{W}(\epsilon)\) for the fine stage; network parameters \(\Theta\)
\ENSURE Trained parameters \(\Theta^\star\)
\FOR{\(\epsilon=1\) to \(T\)}
    \STATE \(\mathbf{F} \gets \text{Encoder}(\mathcal{P}_{\text{in}};\Theta)\)
    \STATE \(\{\mathcal{P}_\text{coarse}^{(k)}\}_{k=1}^{K} \gets \text{SeedGen/Up}(\mathbf{F};\Theta)\) \COMMENT{produce \(K\) coarse clouds}
    \STATE \(\mathcal{P}_\text{fine} \gets \text{Rebuild}(\mathcal{P}_\text{coarse}^{(K)};\Theta)\) \COMMENT{final fine cloud}
    \STATE \textbf{Coarse-stage weighting (static):} \((\alpha_\text{coarse},\beta_\text{coarse})\!\leftarrow\!(\tau,\theta)\)
    \STATE \(\mathcal{L}_{\mathrm{coarse}} \gets \sum_{k=1}^{K} d_{\mathrm{FCD}}\!\big(\mathcal{P}_\text{coarse}^{(k)},\mathcal{G};\alpha_\text{coarse},\beta_\text{coarse},r\big)\)
    \STATE \textbf{Fine-stage weighting (scheduled):} \((\alpha_\text{fine},\beta_\text{fine})\!\leftarrow\!\mathcal{W}(\epsilon)\)
    \STATE \(\mathcal{L}_{\mathrm{fine}} \gets d_{\mathrm{FCD}}\!\big(\mathcal{P}_\text{fine},\mathcal{G};\alpha_\text{fine},\beta_\text{fine},r\big)\)
    \STATE \(\mathcal{L}_{\mathrm{total}} \gets \mathcal{L}_{\mathrm{coarse}} + \mathcal{L}_{\mathrm{fine}}\)
    \STATE \(\Theta \leftarrow \Theta - \eta \nabla_{\Theta}\mathcal{L}_{\mathrm{total}}\)  \COMMENT{any optimizer}
\ENDFOR
\STATE \RETURN \(\Theta^\star\)
\end{algorithmic}
\end{algorithm}

Existing point cloud completion networks often adopt a "coarse-to-fine" multi-stage architecture \cite{yuan2018pcn,xiang2021snowflakenet,tang2022lake,zhou2022seedformer,yu2023adapointr,wei2025pcdreamer}. As shown in Fig.~\ref{fig:frame}, we can seamlessly integrate the FCD into these end-to-end pipelines. An incomplete input point cloud \(\mathcal{P}_{\text{in}}\) is passed through an encoder and a seed generator (supplemented by upsampling/reconstruction modules) to sequentially produce a series of coarse-stage point clouds \(\{\mathcal{P}_\text{coarse}^{(k)}\}_{k=1}^{K}\) and a final fine point cloud \(\mathcal{P}_\text{fine}\). To uniformly supervise all stages, we minimize the following in a single forward-backward pass:
\begin{equation}
\mathcal{L}_{\text{total}} = \mathcal{L}_{\text{coarse}} + \mathcal{L}_{\text{fine}},
\end{equation}
\begin{equation}
\mathcal{L}_{\text{coarse}}=\sum_{k=1}^{K} d_{\text{FCD}}\!\big(\mathcal{P}_\text{coarse}^{(k)},\mathcal{G};\alpha_\text{coarse},\beta_\text{coarse}\big),
\end{equation}
\begin{equation}
\mathcal{L}_{\text{fine}}=d_{\text{FCD}}\!\big(\mathcal{P}_\text{fine},\mathcal{G};\alpha_\text{fine},\beta_\text{fine}\big),
\end{equation}
where \(d_{\text{FCD}}(\mathcal{P},\mathcal{G};\alpha,\beta)\!=\!\alpha\,d_{\text{CD}_\text{{local}}}\!\!+\!\beta\,d_{\text{CD}_{\text{global}}}\), with the two terms corresponding to the local and global components from Section~\ref{para:fcd-explain}. The distance can be either Euclidean or squared Euclidean. 
The weighting strategy is consistent with Section~\ref{para:weighting-strategy}: the coarse stages adopt the static scheme to prioritize establishing the global structure, with fixed weights \(\alpha_\text{coarse}=\tau, \beta_\text{coarse}=\theta\ (\theta>\tau)\). To achieve a better trade-off in the fine stage, we use one of the weighting strategies from Fig.~\ref{fig:frame}, \((\alpha_\text{fine},\beta_\text{fine})=\mathcal{W}(\epsilon)\in\{\text{Static,Stair,Linear Abridged Linear,Exponential,Uncertainty}\}\). The Uncertainty strategy is initialized with \((\tau,\theta)\) and adapts during training. If a baseline model contains only a single coarse output, it is treated as a degenerate case with \(K=1\). The complete procedure is detailed in Algorithm~\ref{alg:fcd_training}.
\section{Experiments}
\label{results}

\begin{figure}[t]
    \centering
    \includegraphics[width=\linewidth]{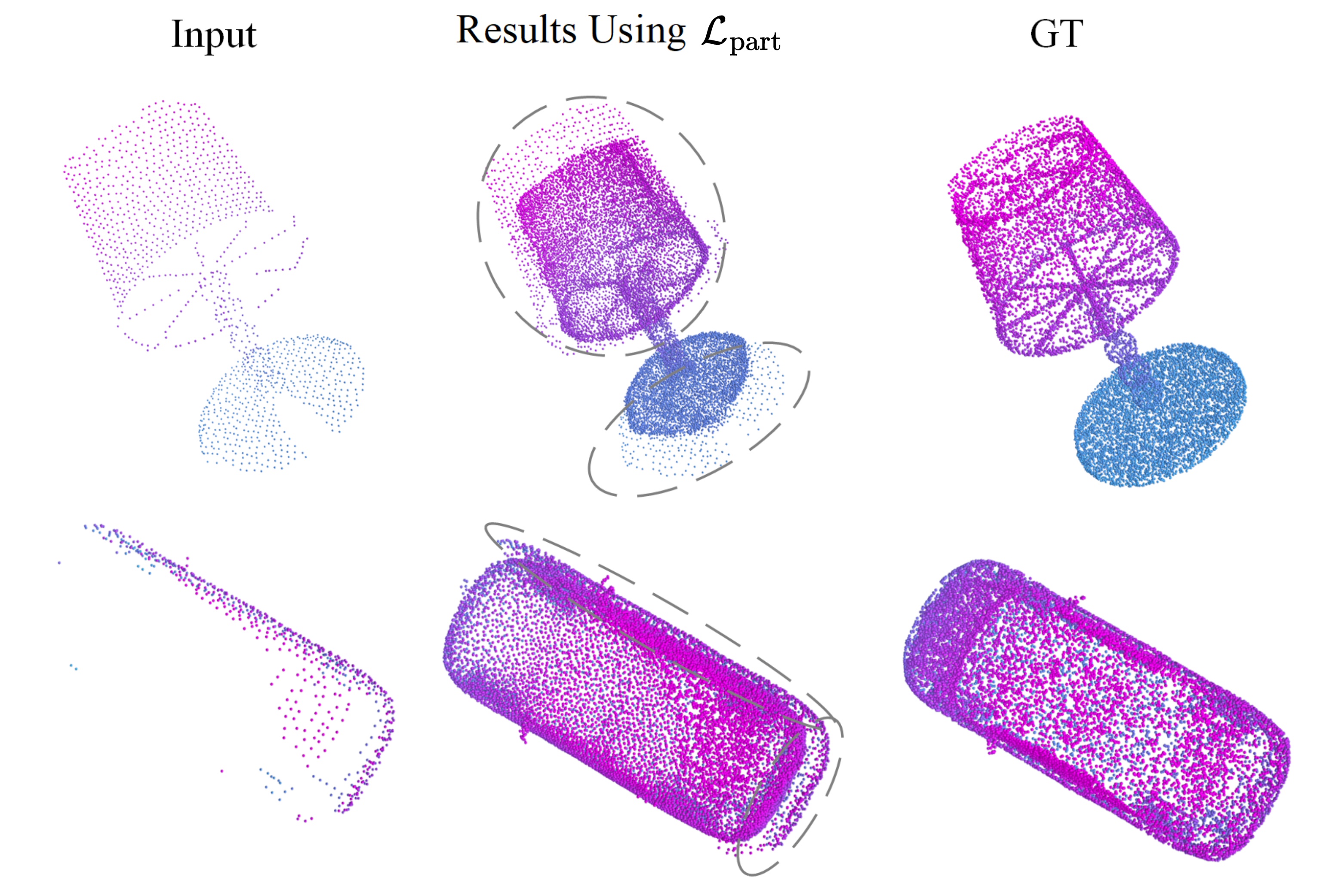}
    \caption{Visualization of 'ghosting' artifacts (dashed circles) caused by the \(\mathcal{L}_{\text{part}}\) loss term, justifying its exclusion.}
    \label{fig:ghost}
\end{figure}

To comprehensively evaluate the effectiveness of our proposed FCD as an objective function, we conducted extensive experiments. This section first introduces the experimental setup, including the datasets, evaluation metrics, and implementation details. Subsequently, we present the main results on mainstream benchmarks. We further substantiate our findings through ablation studies on weighting strategies, qualitative visualizations and several generalization experiments. Finally, we analyze the failure case and computational complexity.

\subsection{Experimental Setup}

\subsubsection{Datasets}
The \textbf{ShapeNet55 dataset} \cite{yu2021pointr}, derived from ShapeNet \cite{chang2015shapenet}, includes all 55 categories. It contains 41,952 shapes for training and 10,518 for testing. Complete point clouds contain 8,192 points, and input point clouds contain 2,048 points. Following~\cite{yu2021pointr}, we select 8 fixed viewpoints and set the number of incomplete points to 2,048, 4,096, and 6,144 (25\%, 50\%, and 75\% of the complete cloud), representing easy, medium, and hard difficulty levels during testing.
The \textbf{PCN dataset} \cite{yuan2018pcn} is a widely-used completion benchmark, also derived from ShapeNet, including 8 categories. We follow the original PCN split, using 28,974 samples for training and 1,200 for testing. The complete point cloud (16,384 points) is sampled from the mesh surface, and the incomplete input is generated via back-projection from 8 different viewpoints.
The \textbf{KITTI dataset} \cite{geiger2013kitti} comprises real-world car point clouds extracted from outdoor LiDAR scans using 3D bounding boxes, totaling 2,401 sparse objects. Unlike synthetic datasets, KITTI scans can be highly sparse and lack ground truth. It is primarily used to evaluate model generalization to sparse, real-world data in an unsupervised setting.
The \textbf{ABC dataset} \cite{koch2019abc} is a large-scale Computer-Aided Design (CAD) model dataset with over one million models from real engineering projects. These models feature complex geometric structures and precise surfaces. We selected 5 engineering workpieces to validate FCD in industrial application scenarios. We generated incomplete data by simulating real-world acquisition processes. We generated 2,000 training and 100 test samples per workpiece. The complete cloud has 16,384 points, and the incomplete input has 2,048 points.
The \textbf{PU-GAN dataset} \cite{li2019pu} is commonly used for point cloud upsampling, containing 147 3D models (120 for training, 27 for testing). For each training model, 24,000 pairs are generated using Poisson disk sampling, with an input of 256 points and a ground truth of 1,024 points. Each test model provides one input/ground-truth pair. Test inputs have 2,048 points, while the ground truth has 8,192 or 32,768 points (4x and 16x upsampling).

\subsubsection{Evaluation Metrics}
We use a multi-dimensional evaluation, as a single metric cannot fully capture the quality of generated point clouds. For supervised point cloud completion (ShapeNet, PCN, ABC), we employ CD to mainly measure local detail fitting accuracy; EMD to assess overall distributional similarity; DCD to quantify global uniformity; and F-Score to consider precision and recall. For the ABC dataset, we also introduce Point-to-Mesh distance (P2F)~\cite{berger2013benchmark} to evaluate fidelity to the original CAD surfaces. For unsupervised completion on KITTI, which lacks ground truth, we adopt Fidelity~\cite{yuan2018pcn} (consistency with the partial input), Minimum Matching Distance~\cite{yuan2018pcn} (MMD, assessing shape similarity), and Consistency~\cite{yuan2018pcn} (quantifying structural plausibility). For the upsampling task (PU-GAN), we incorporate Hausdorff Distance (HD)~\cite{li2019pu}, which measures the maximum mismatch, in addition to the metrics above.

\subsubsection{Implementation Details}

We validate FCD on two state-of-the-art completion networks, AdaPoinTr \cite{yu2023adapointr} and SeedFormer \cite{zhou2022seedformer}, and one upsampling network, RepKPU \cite{rong2024repkpu}. All experiments were conducted on servers equipped with two NVIDIA 3090 (24G) or two NVIDIA 4090 (24G) GPUs. We adhere to the original training parameters whenever possible. 
In all FCD experiments, the weight bounds are \(\theta=2\) and \(\tau=1\), the transition step is \(t=200\), and the exponential decay rate is \(\sigma=200\).
For \textbf{AdaPoinTr}, we use the AdamW \cite{loshchilov2018fixing} optimizer with an initial learning rate of 0.0001 and weight decay of 0.0005. The learning rate is decayed by 0.9 every 21 epochs. We train for 600 epochs on ShapeNet55 (batch size 48) and 400 epochs on PCN (batch size 16). 
For \textbf{SeedFormer}, we use the Adam \cite{kingma2014adam} optimizer with an initial learning rate of 0.001, decayed by 0.1 every 100 epochs. We train for 400 epochs on both PCN (batch size 16) and ShapeNet55 (batch size 48). Notably, the original SeedFormer paper used an additional partial matching (PM) loss term \(\mathcal{L}_{\text{part}}\) (matching the input partial cloud to the prediction). Our experiments found this term can harm performance on simpler datasets (Table~\ref{tab:seedformer-pcn}) and is prone to ghosting artifacts (Fig.~\ref{fig:ghost}). Therefore, to ensure a fair comparison of the core objective functions, we excluded it from all experiments. 
And for \textbf{RepKPU}, we also use the Adam optimizer with an initial learning rate of 0.001, gradually decayed to 0.0001. We train for 100 epochs with a batch size of 32.

\begin{figure*}[t]
    \centering
    \includegraphics[width=\linewidth]{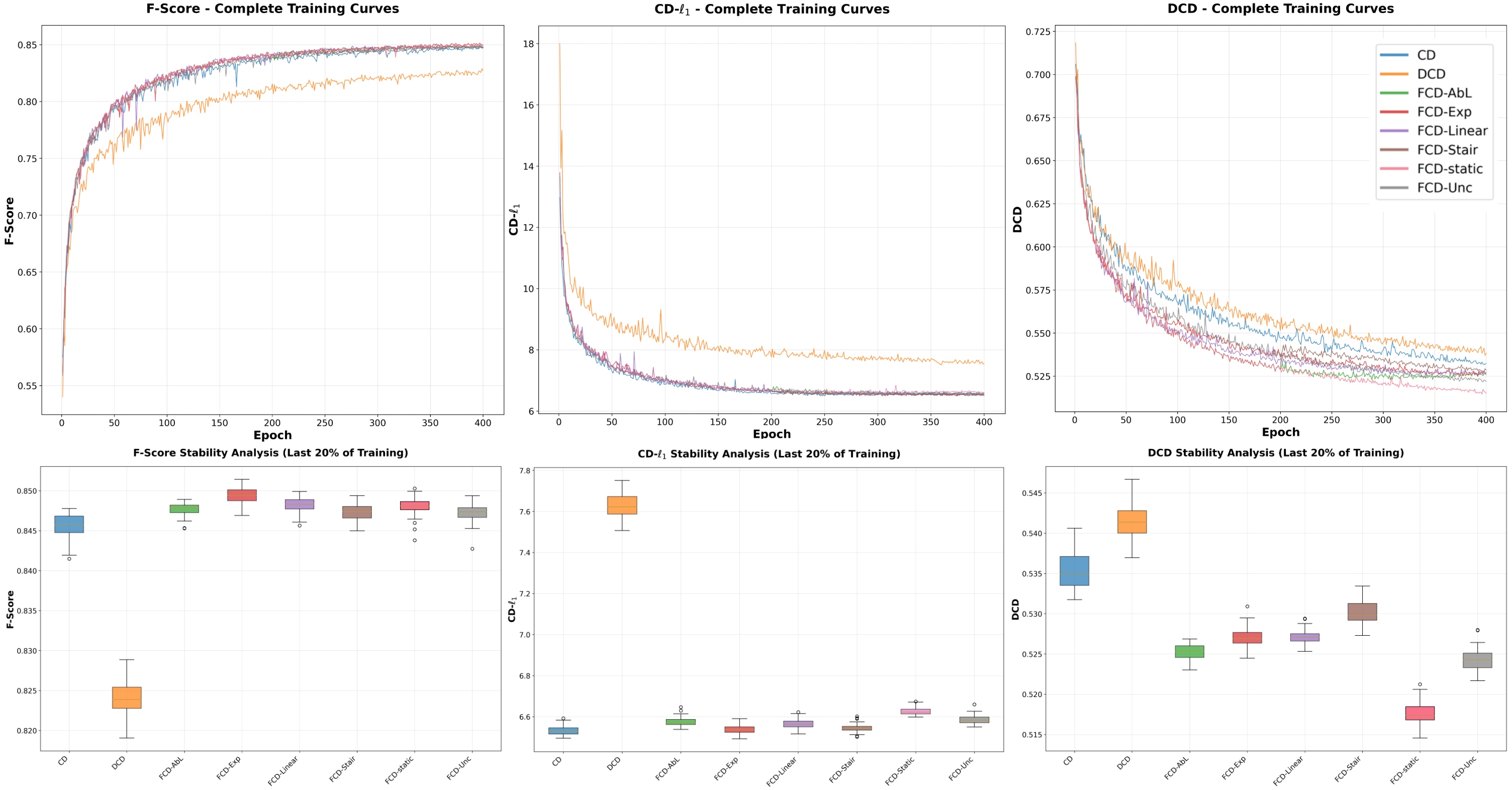}
    \caption{Training dynamics (F-Score/CD-\(\ell_1\)/DCD) and stability (box plots) for AdaPoinTr on PCN. Top: FCD variants achieve superior DCD and F-Score curves. Bottom: Box plots show FCD's better median convergence and lower variance .}
    \label{fig:loss-comparision}
\end{figure*}

\subsection{Effectiveness on ShapeNet55 Dataset}
We first validate FCD's effectiveness on the large-scale ShapeNet55 benchmark. Experiments in this section utilize the static FCD (Static) (\(\alpha=1, \beta=2\)) and compare it against baseline models trained with standard \(\text{CD-}\ell_2\) (\(\alpha=1, \beta=1\)). Following \cite{yu2023adapointr,zhou2022seedformer}, all experiments use squared Euclidean distance (\(\ell_2\) version). We omit the additional \(\mathcal{L}_\text{part}\) loss term introduced in \cite{zhou2022seedformer}. Consequently, our reproduced baseline results may differ from those in the original paper. All results are reported as the mean and standard deviation over three independent trials.

\subsubsection{Results on AdaPoinTr}
We strictly follow the original AdaPoinTr setup, only replacing the \(\text{CD-}\ell_2\) loss with static FCD. As shown in Table~\ref{tab:adapointr-shapenet}, FCD demonstrates a clear advantage in optimizing global structure. The most significant improvement is in global distribution quality: FCD achieves a substantial reduction in the average DCD metric (from 0.613 to 0.537), indicating FCD guides the model to generate more uniform and structurally complete point clouds. This global enhancement is achieved without significantly compromising local fidelity. Although the EMD metric indicates a slight trade-off (28.983 \(\rightarrow\) 29.253), other key local metrics, such as the F-Score, show consistent improvement across all difficulty levels. Furthermore, FCD exhibits superior optimization stability: the average standard deviations for DCD and F-Score are reduced by nearly an order of magnitude (e.g., DCD std: 0.043 \(\rightarrow\) 0.006), confirming FCD provides a more stable optimization gradient.

\subsubsection{Results on SeedFormer}
We conducted parallel experiments on the SeedFormer architecture. As shown in Table~\ref{tab:seedformer-shapenet}, the results further corroborate the generalizability of FCD. Consistent with findings on AdaPoinTr, FCD on SeedFormer also shows a significant DCD advantage, outperforming the baseline at all difficulty levels (Avg. DCD: 0.596 \(\rightarrow\) 0.528). This again confirms FCD's effectiveness in mitigating point clustering and structural voids. However, on this architecture, the improvement in global uniformity is accompanied by a clearer metric trade-off: we observe slight performance drops in CD-\(\ell_2\), EMD, and F-Score (e.g., Avg. F-Score: 0.412 \(\rightarrow\) 0.403). Despite this trade-off, FCD's ability to consistently and significantly improve the DCD metric across two mainstream networks strongly validates its efficacy as a plug-and-play objective function for addressing the bottleneck of global distribution quality.

% \usepackage{multirow}
\begin{table}[t]
\centering
\caption{RESULTS OF ADAPOINTR ON SHAPENET55 DATASET}
\resizebox{\linewidth}{!}{%
\begin{threeparttable}
\begin{tabular}{@{}cl|llll@{}}
\toprule
\multicolumn{2}{c|}{}                                               & CD-\(\ell_2\)↓               & EMD↓                  & DCD↓                 & F1↑                  \\ \midrule
\multicolumn{1}{c|}{\multirow{4}{*}{AdaPoinTr+CD}}         & Easy   & \textbf{0.492±\sd{0.015}} & 27.324±\sd{1.100}          & 0.588±\sd{0.050}          & 0.427±\sd{0.020}          \\
\multicolumn{1}{c|}{}                                      & Median & \textbf{0.697±\sd{0.013}} & 27.860±\sd{0.341}          & 0.601±\sd{0.048}          & 0.411±\sd{0.020}          \\
\multicolumn{1}{c|}{}                                      & Hard   & \textbf{1.258±\sd{0.034}} & \textbf{31.763±\sd{2.066}} & 0.652±\sd{0.032}          & 0.369±\sd{0.015}          \\
\multicolumn{1}{c|}{}                                      & Avg.   & \textbf{0.816±\sd{0.019}} & \textbf{28.983±\sd{0.221}} & 0.613±\sd{0.043}          & 0.402±\sd{0.018}          \\ \midrule
\multicolumn{1}{c|}{\multirow{4}{*}{AdaPoinTr+FCD (Static)}} & Easy   & 0.500±\sd{0.003}          & \textbf{25.197±\sd{0.309}}* & \textbf{0.502±\sd{0.005}}* & \textbf{0.429±\sd{0.002}} \\
\multicolumn{1}{c|}{}                                      & Median & 0.716±\sd{0.008}          & \textbf{26.912±\sd{0.470}}* & \textbf{0.517±\sd{0.007}}* & \textbf{0.415±\sd{0.002}} \\
\multicolumn{1}{c|}{}                                      & Hard   & 1.286±\sd{0.012}          & 35.651±\sd{0.910}          & \textbf{0.592±\sd{0.007}}* & \textbf{0.384±\sd{0.003}} \\
\multicolumn{1}{c|}{}                                      & Avg.   & 0.834±\sd{0.007}          & 29.253±\sd{0.548}          & \textbf{0.537±\sd{0.006}}* & \textbf{0.409±\sd{0.002}} \\ \bottomrule
\end{tabular}
\begin{tablenotes}
\footnotesize
\item Overall results on 55 categories for three difficulty degrees under CD-\(\ell_2\) (\(\times 10^3\)), EMD (\(\times 10^3\)), DCD and F-Score@1\%. We report Mean ± Std Dev over 3 runs. * denotes a statistically significant improvement over the CD baseline (\(p < 0.05\), one-tailed paired t-test).
\end{tablenotes}
\end{threeparttable}
}

\label{tab:adapointr-shapenet}
\end{table}

\begin{table}[t]
\centering
\caption{RESULTS OF SEEDFOMER ON SHAPENET55 DATASET}
\resizebox{\linewidth}{!}{%
\begin{threeparttable}
\begin{tabular}{@{}cl|llll@{}}
\toprule
\multicolumn{2}{c|}{}                                                  & \multicolumn{1}{c}{CD-\(\ell_2\)↓} & \multicolumn{1}{c}{EMD↓} & \multicolumn{1}{c}{DCD↓} & \multicolumn{1}{c}{F1↑} \\ \midrule
\multicolumn{1}{c|}{\multirow{4}{*}{SeedFormer+CD}}           & Easy   & \textbf{0.480\sd{±0.007}}       & 24.633\sd{±0.370}             & 0.559\sd{±0.001}              & \textbf{0.457\sd{±0.002}}    \\
\multicolumn{1}{c|}{}                                         & Median & \textbf{0.717\sd{±0.008}}       & \textbf{25.528\sd{±0.439}}    & 0.584\sd{±0.003}              & \textbf{0.421\sd{±0.003}}    \\
\multicolumn{1}{c|}{}                                         & Hard   & \textbf{1.366\sd{±0.009}}       & \textbf{33.132\sd{±0.303}}    & 0.643\sd{±0.003}              & \textbf{0.359\sd{±0.002}}    \\
\multicolumn{1}{c|}{}                                         & Avg.   & \textbf{0.854\sd{±0.008}}       & \textbf{27.765\sd{±0.361}}    & 0.596\sd{±0.002}              & \textbf{0.412\sd{±0.002}}    \\ \midrule
\multicolumn{1}{c|}{\multirow{4}{*}{SeedFormer+FCD (Static)}} & Easy   & 0.504\sd{±0.001}                & \textbf{23.192\sd{±0.198}}*    & \textbf{0.481\sd{±0.001}}*      & 0.449\sd{±0.002}             \\
\multicolumn{1}{c|}{}                                         & Median & 0.764\sd{±0.002}                & 25.889\sd{±0.372}             & \textbf{0.515\sd{±0.003}}*     & 0.411\sd{±0.002}             \\
\multicolumn{1}{c|}{}                                         & Hard   & 1.451\sd{±0.012}                & 38.067\sd{±0.448}             & \textbf{0.587\sd{±0.004}}*     & 0.348\sd{±0.001}             \\
\multicolumn{1}{c|}{}                                         & Avg.   & 0.906\sd{±0.005}                & 29.050\sd{±0.267}             & \textbf{0.528\sd{±0.003}}*     & 0.403\sd{±0.002}             \\ \bottomrule
\end{tabular}
\begin{tablenotes}
\footnotesize
\item Overall results on 55 categories for three difficulty degrees under CD-\(\ell_2\) (\(\times 10^3\)), EMD (\(\times 10^3\)), DCD and F-Score@1\%. We report Mean ± Std Dev over 3 runs. * denotes a statistically significant improvement over the CD baseline (\(p < 0.05\), one-tailed paired t-test).
\end{tablenotes}
\end{threeparttable}
}
\label{tab:seedformer-shapenet}
\end{table}

\subsection{Ablation Study on PCN Dataset}

\begin{figure*}[ht]
    \centering
    \includegraphics[width=0.95\linewidth]{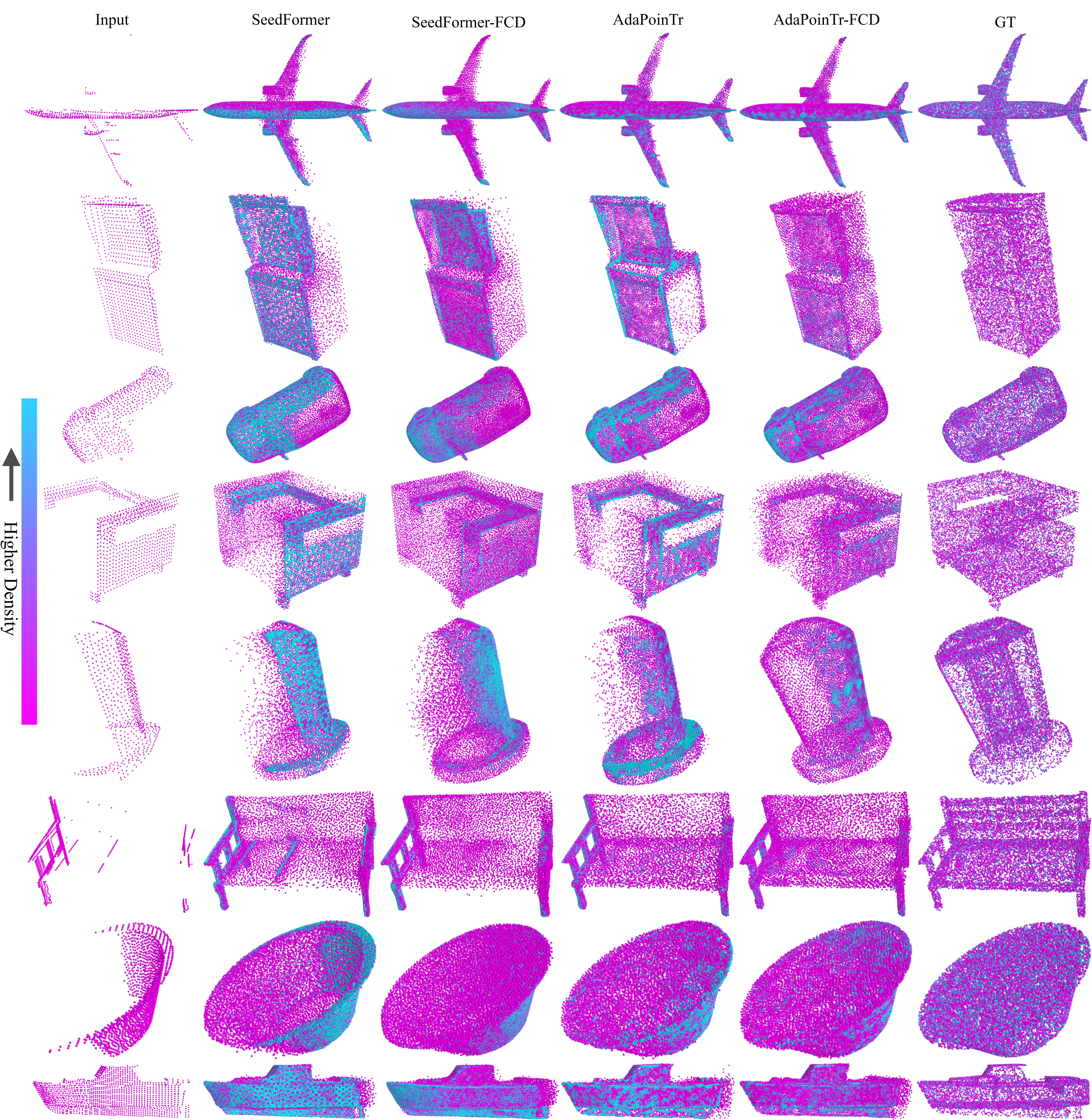}
    \caption{Visual results on PCN dataset. Points are colored based on local density, ranging from purple (low density) to blue (high density). The results show that FCD-guided methods (SeedFormer-FCD, AdaPoinTr-FCD) generate globally more uniform point clouds and effectively mitigate the local clustering artifacts (indicated by blue regions) observed in the standard CD baseline methods.}
    \label{fig:pcn-res}
\end{figure*}

\label{quanti}
We conducted ablation studies on the PCN dataset with two primary goals: first, to validate that the core FCD principle (\(\beta > \alpha\)) consistently yields improvements in global distribution metrics (e.g., DCD and EMD) compared to the standard CD baseline; and second, to analyze how these different weighting strategies navigate the resulting trade-off between these global improvements and their impact on local metrics (eg., CD, F-Score). We also include DCD as an objective function for comparison. All CD and FCD experiments in this section use Euclidean distance (\(\ell_1\) version). All results are reported as the mean and standard deviation over three trials. 

\subsubsection{Results on AdaPoinTr}
\label{para:ada-pcn}
Table~\ref{tab:adapointr-pcn} summarizes the performance comparison. Using DCD as a standalone loss fails to yield balanced improvements, significantly degrading F-Score. In contrast, all FCD variants consistently outperform the CD baseline on global metrics (EMD and DCD), validating our first objective. This global enhancement, however, necessitates navigating a trade-off with the \(\text{CD-}\ell_1\) metric, highlighting our second objective. The preset Static strategy exemplifies an aggressive global optimization, achieving the largest reductions in EMD (23.790 \(\rightarrow\) 21.395) and DCD (0.533 \(\rightarrow\) 0.514). While this incurs a slight \(\text{CD-}\ell_1\) cost (6.509 \(\rightarrow\) 6.591), it notably secures the best F-Score (0.850), suggesting a superior local optimization. Conversely, adaptive strategies like Linear and Stair represent a more conservative balance, improving global metrics while keeping \(\text{CD-}\ell_1\) scores (e.g., Linear: 6.538) much closer to the baseline. Notably, Uncertainty Weighting, requiring no manual schedule, also achieves highly competitive results (DCD: 0.520, EMD: 22.102), second only to the Static strategy. Fig.~\ref{fig:loss-comparision} visualizes this: FCD training curves (top) are consistently below the CD baseline for DCD and above for F-Score. Furthermore, box plots (bottom) for the last 20\% of epochs show FCD variants have better median values and more compact distributions, validating their superior convergence and stability.

\begin{table*}[t]
\centering
\caption{RESULTS OF ADAPOINTR ON PCN DATASET}
\resizebox{\textwidth}{!}{%
\begin{threeparttable}
\begin{tabular}{@{}ccl|cccccccc|llll@{}}
\toprule
\multicolumn{3}{c|}{}                                                                                                                     & Plane       & Cabinet     & Car         & Chair       & Lamp        & Couch       & Table       & Boat        & DCD↓                 & CD-\(\ell_1\)↓               & EMD↓                  & F-Score@ 1\%↑        \\ \midrule
\multicolumn{3}{c|}{AdaPoinTr + CD}                                                                                                       & 0.504\sd{±0.004} & 0.546\sd{±0.003} & 0.560\sd{±0.003} & 0.514\sd{±0.005} & 0.544\sd{±0.007} & 0.568\sd{±0.006} & 0.472\sd{±0.003} & 0.559\sd{±0.003} & 0.533\sd{±0.003}          & \textbf{6.509\sd{±0.019}} & 23.790\sd{±0.294}          & 0.847\sd{±0.002}          \\
\multicolumn{3}{c|}{AdaPoinTr + DCD}                                                                                                      & 0.519\sd{±0.001} & 0.542\sd{±0.006} & 0.550\sd{±0.008} & 0.526\sd{±0.003} & 0.557\sd{±0.009} & 0.567\sd{±0.009} & 0.482\sd{±0.002} & 0.565\sd{±0.003} & 0.538\sd{±0.002}          & 7.480\sd{±0.150}          & 22.331\sd{±1.498}          & 0.827\sd{±0.002}          \\ \midrule
\multicolumn{1}{c|}{\multirow{6}{*}{AdaPoinTr + FCD}} & \multicolumn{1}{c|}{\multirow{5}{*}{Preset Adaptive Weighting}} & Static          & 0.487\sd{±0.001} & 0.527\sd{±0.001} & 0.536\sd{±0.001} & 0.500\sd{±0.002} & 0.522\sd{±0.003} & 0.544\sd{±0.002} & 0.461\sd{±0.001} & 0.536\sd{±0.002} & \textbf{0.514\sd{±0.001}}* & 6.591\sd{±0.011}      & \textbf{21.395\sd{±0.190}}* & \textbf{0.850\sd{±0.001}}* \\
\multicolumn{1}{c|}{}                                 & \multicolumn{1}{c|}{}                                           & Stair           & 0.500\sd{±0.002} & 0.540\sd{±0.002} & 0.553\sd{±0.002} & 0.511\sd{±0.002} & 0.533\sd{±0.002} & 0.561\sd{±0.002} & 0.469\sd{±0.002} & 0.551\sd{±0.001} & 0.527\sd{±0.001}*          & 6.526\sd{±0.023}          & 23.093\sd{±0.214}          & 0.849\sd{±0.002}*          \\
\multicolumn{1}{c|}{}                                 & \multicolumn{1}{c|}{}                                           & Linear          & 0.496\sd{±0.003} & 0.539\sd{±0.001} & 0.551\sd{±0.003} & 0.510\sd{±0.002} & 0.534\sd{±0.003} & 0.558\sd{±0.004} & 0.466\sd{±0.001} & 0.550\sd{±0.001} & 0.526\sd{±0.001}*          & 6.538\sd{±0.022}          & 22.807\sd{±0.182}*          & 0.849\sd{±0.001}          \\
\multicolumn{1}{c|}{}                                 & \multicolumn{1}{c|}{}                                           & Abridged Linear & 0.494\sd{±0.002} & 0.536\sd{±0.001} & 0.548\sd{±0.002} & 0.508\sd{±0.002} & 0.531\sd{±0.004} & 0.556\sd{±0.002} & 0.468\sd{±0.002} & 0.546\sd{±0.002} & 0.523\sd{±0.001}*          & 6.538\sd{±0.018}          & 22.571\sd{±0.230}*          & 0.849\sd{±0.001}          \\
\multicolumn{1}{c|}{}                                 & \multicolumn{1}{c|}{}                                           & Exponential     & 0.495\sd{±0.005} & 0.538\sd{±0.003} & 0.549\sd{±0.003} & 0.509\sd{±0.004} & 0.532\sd{±0.003} & 0.557\sd{±0.003} & 0.467\sd{±0.001} & 0.547\sd{±0.002} & 0.524\sd{±0.002}*          & 6.540\sd{±0.049}          & 22.700\sd{±0.180}*          & 0.849\sd{±0.002}*          \\ \cmidrule(l){2-15} 
\multicolumn{1}{c|}{}                                 & \multicolumn{2}{c|}{Uncertainty   Weighting}                                      & 0.490\sd{±0.002} & 0.532\sd{±0.002} & 0.543\sd{±0.004} & 0.505\sd{±0.002} & 0.532\sd{±0.004} & 0.552\sd{±0.003} & 0.464\sd{±0.003} & 0.541\sd{±0.002} & 0.520\sd{±0.002}*          & 6.608\sd{±0.055}          & 22.102\sd{±0.414}*          & 0.848\sd{±0.002}          \\ \bottomrule
\end{tabular}

\begin{tablenotes}
\footnotesize
\item Detailed DCD results (per-category and overall), and overall results for CD-\(\ell_1\) (\(\times 10^3\)), EMD (\(\times 10^3\)), and F-Score@1\%. We report Mean ± Std Dev over 3 runs. * denotes a statistically significant improvement over the CD baseline (\(p < 0.05\), one-tailed paired t-test).
\end{tablenotes}
\end{threeparttable}

}
\label{tab:adapointr-pcn}
\end{table*}

\subsubsection{Results on SeedFormer}
Table~\ref{tab:seedformer-pcn} reports the ablation results on SeedFormer (using SeedFormer + CD as the baseline). The conclusions are highly consistent with those from AdaPoinTr (Section~\ref{para:ada-pcn}), demonstrating FCD's generalization capability. The preset Static strategy again shows the strongest performance, prioritizing global structure to achieve best-in-class EMD (improving by 13.22\% to 22.797), DCD (0.513), and F-Score (0.837). As expected, this advantage is accompanied by a slight trade-off in the \(\text{CD-}\ell_1\) metric (6.617 \(\rightarrow\) 6.711). The dynamic strategies further clarify this trade-off. Methods like Linear and Exponential showcase an alternative balance, achieving \(\text{CD-}\ell_1\) metrics (e.g., Linear: 6.624) nearly on par with the baseline. This mitigates the local fitting cost while still maintaining clear superiority over CD in all other metrics. This result re-confirms that FCD provides an effective, generalizable optimization and offers a tunable balance between global uniformity and local fitting.

\begin{table*}[t]
\centering
\caption{RESULTS OF SEEDFOMER ON PCN DATASET}
\resizebox{\textwidth}{!}{%
\begin{threeparttable}
\begin{tabular}{@{}ccl|cccccccc|llll@{}}
\toprule
\multicolumn{3}{c|}{}                                                                                                                    & Plane                & Cabinet              & Car                  & Chair                & Lamp                 & Couch                & Table                & Boat                 & DCD↓                 & CD-\(\ell_1\)↓               & EMD↓                  & F-Score@ 1\%↑        \\ \midrule
\multicolumn{3}{c|}{SeedFormer+CD+PM}                                                                                                    & 0.540\sd{±0.007}          & 0.585\sd{±0.002}          & 0.590\sd{±0.002}          & 0.563\sd{±0.005}          & 0.557\sd{±0.005}          & 0.620\sd{±0.004}          & 0.510\sd{±0.003}          & 0.585\sd{±0.006}          & 0.569\sd{±0.003}          & \textbf{6.757\sd{±0.025}} & 28.156\sd{±0.383}          & 0.819\sd{±0.002}          \\
\multicolumn{3}{c|}{SeedFormer+CD}                                                                                                       & 0.511\sd{±0.008}          & 0.547\sd{±0.004}          & 0.561\sd{±0.001}          & 0.537\sd{±0.004}          & 0.536\sd{±0.003}          & 0.597\sd{±0.004}          & 0.469\sd{±0.005}          & 0.567\sd{±0.005}          & 0.540\sd{±0.003}          & \textbf{6.617\sd{±0.007}} & 26.269\sd{±0.529}          & 0.829\sd{±0.001}          \\ \midrule
\multicolumn{1}{c|}{\multirow{6}{*}{SeedFormer+FCD}} & \multicolumn{1}{c|}{\multirow{5}{*}{Preset Adaptive Weighting}} & Static          & \textbf{0.488\sd{±0.002}} & \textbf{0.515\sd{±0.001}} & \textbf{0.530\sd{±0.001}} & \textbf{0.511\sd{±0.000}} & \textbf{0.512\sd{±0.000}} & \textbf{0.560\sd{±0.002}} & \textbf{0.449\sd{±0.001}} & \textbf{0.537\sd{±0.001}} & \textbf{0.513\sd{±0.000}}* & 6.711\sd{±0.024}          & \textbf{22.797\sd{±0.116}}* & \textbf{0.837\sd{±0.000}}* \\
\multicolumn{1}{c|}{}                                & \multicolumn{1}{c|}{}                                           & Stair           & 0.502\sd{±0.003}          & 0.538\sd{±0.007}          & 0.549\sd{±0.002}          & 0.529\sd{±0.001}          & 0.532\sd{±0.004}          & 0.584\sd{±0.002}          & 0.464\sd{±0.003}          & 0.557\sd{±0.003}          & 0.532\sd{±0.003}          & 6.645\sd{±0.012}          & 24.260\sd{±0.237}*          & 0.832\sd{±0.002}          \\
\multicolumn{1}{c|}{}                                & \multicolumn{1}{c|}{}                                           & Linear          & 0.500\sd{±0.002}          & 0.536\sd{±0.003}          & 0.548\sd{±0.001}          & 0.526\sd{±0.001}          & 0.527\sd{±0.001}          & 0.582\sd{±0.001}          & 0.461\sd{±0.000}          & 0.556\sd{±0.002}          & 0.529\sd{±0.001}*          & 6.624\sd{±0.029}          & 24.393\sd{±0.279}*          & 0.833\sd{±0.001}*          \\
\multicolumn{1}{c|}{}                                & \multicolumn{1}{c|}{}                                           & Abridged Linear & 0.501\sd{±0.001}          & 0.530\sd{±0.004}          & 0.547\sd{±0.002}          & 0.526\sd{±0.002}          & 0.528\sd{±0.003}          & 0.583\sd{±0.005}          & 0.460\sd{±0.001}          & 0.554\sd{±0.002}          & 0.529\sd{±0.001}*          & 6.652\sd{±0.012}          & 24.027\sd{±0.061}*          & 0.833\sd{±0.001}*         \\
\multicolumn{1}{c|}{}                                & \multicolumn{1}{c|}{}                                           & Exponential     & 0.498\sd{±0.004}          & 0.535\sd{±0.003}          & 0.546\sd{±0.001}          & 0.524\sd{±0.003}          & 0.525\sd{±0.002}          & 0.582\sd{±0.001}          & 0.460\sd{±0.004}          & 0.554\sd{±0.002}          & 0.528\sd{±0.002}*          & 6.636\sd{±0.052}          & 24.299\sd{±0.177}*          & 0.833\sd{±0.002}*          \\ \cmidrule(l){2-15} 
\multicolumn{1}{c|}{}                                & \multicolumn{2}{c|}{Uncertainty   Weighting}                                      & 0.491\sd{±0.005}          & 0.526\sd{±0.009}          & 0.535\sd{±0.009}          & 0.514\sd{±0.007}          & 0.519\sd{±0.007}          & 0.565\sd{±0.008}          & 0.451\sd{±0.005}          & 0.541\sd{±0.006}          & 0.517\sd{±0.006}*          & 6.726\sd{±0.028}          & 23.390\sd{±1.175}*          & 0.835\sd{±0.003}*          \\ \bottomrule
\end{tabular}

\begin{tablenotes}
\footnotesize
\item Detailed DCD results (per-category and overall), and overall results for CD-\(\ell_1\) (\(\times 10^3\)), EMD (\(\times 10^3\)), and F-Score@1\%. We report Mean ± Std Dev over 3 runs. * denotes a statistically significant improvement over the CD baseline (\(p < 0.05\), one-tailed paired t-test).
\end{tablenotes}
\end{threeparttable}
}
\label{tab:seedformer-pcn}
\end{table*}

\subsection{Qualitative Analysis}
We provide qualitative results for point cloud completion guided by CD and FCD in Fig.~\ref{fig:pcn-res}. To intuitively illustrate the quality improvement from FCD, we color the points based on local density from low (purple/warm) to high (blue/cold). A visual comparison reveals that when using CD, the generated point clouds exhibit apparent local point clustering in some regions, which consequently leads to structural discontinuities. In contrast, the point clouds generated by FCD are globally more uniform and complete, with visual results that better conform to the true geometry. For categories with complex geometric details (e.g., 2nd and 5th rows in Fig.~\ref{fig:pcn-res}), CD is more prone to generating structurally incorrect or incomplete shapes. FCD, by prioritizing the correctness of the global structure, provides a more reliable foundation for subsequent detail recovery, thus producing higher-quality results. An interesting observation is that although AdaPoinTr generally outperforms SeedFormer in quantitative metrics, in the visual results for some categories (e.g., 4th and 7th rows), SeedFormer, which has more intermediate refinement stages, exhibits slightly superior global uniformity under FCD's guidance. This suggests a potential synergistic effect between FCD's optimization and the network's structural design.

\subsection{Generalization and Application}
To examine FCD's generalization capability across different scenarios, tasks, and data domains, we conducted three extension experiments.

\subsubsection{Real-world Scenarios}
To evaluate FCD's generalization and robustness in real-world scenarios, we follow \cite{yu2023adapointr} to finetune our trained model on PCNCars \cite{yuan2018pcn} and evaluate on the KITTI dataset. As shown in Table~\ref{tab:kitti}, AdaPoinTr trained with FCD outperforms the CD baseline on the key Fidelity and Consistency metrics. This indicates the completion results from FCD not only maintain higher fidelity to the original partial input but also possess a more plausible and consistent internal geometry. Furthermore, the visual comparison in Fig.~\ref{fig:kitti} shows that the model guided by standard CD (middle column) tends to produce completions with uneven density, exhibiting more severe local clustering (higher blue intensity). In contrast, the FCD-guided model (right column) generates point clouds that are significantly more uniform in global distribution and more structurally complete, recovering a more reasonable vehicle shape.

\begin{figure}[t]
    \centering
    \includegraphics[width=\linewidth]{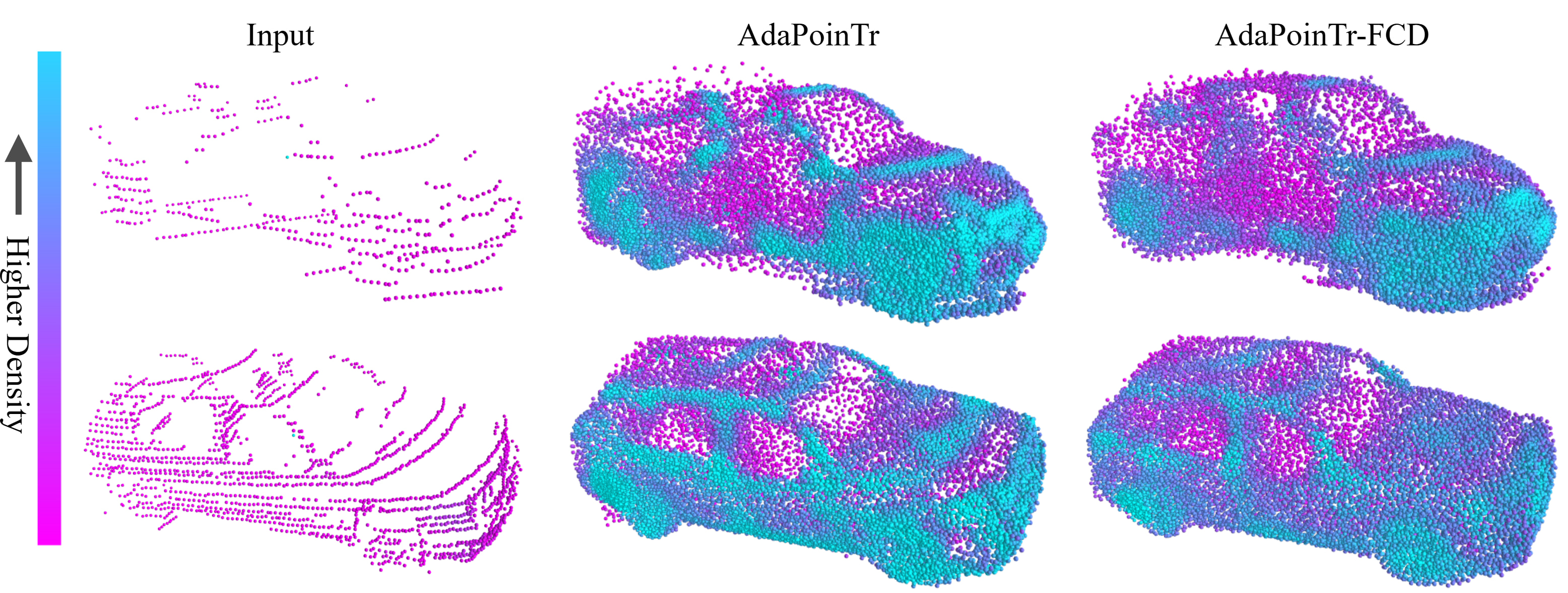}
    \caption{Visual results on the real-world KITTI dataset. Compared to the baseline (AdaPoinTr), the FCD-guided model produces structurally more complete vehicles with a more uniform point distribution, whereas the baseline tends to generate results with uneven density and severe local clustering.}
    \label{fig:kitti}
\end{figure}

\begin{table}[t]
\centering
\caption{RESULTS OF ADAPOINTR ON KITTI DATASET}
\resizebox{0.8\linewidth}{!}{%
\begin{threeparttable}
\begin{tabular}{@{}l|ccc@{}}
\toprule
                       & Fidelity ↓ & MMD ↓ & Consistency↓ \\ \midrule
AdaPoinTr+CD           & 1.439      & \textbf{0.366} & 0.569        \\
AdaPoinTr+FCD (Static) & \textbf{1.409}      & 0.369 & \textbf{0.525}        \\ \bottomrule
\end{tabular}

\begin{tablenotes}
\footnotesize
\item We follow the previous work to finetune the model on PCNCars.
\end{tablenotes}
\end{threeparttable}

}
\label{tab:kitti}
\end{table}

\subsubsection{Engineering Applications}
\label{para:abc}
To assess FCD's effectiveness in industrial application scenarios involving complex topologies and high-precision requirements, we benchmarked on the ABC dataset. These CAD models represent real-world engineering components. As shown in Table~\ref{tab:abc}, FCD again demonstrates its strong capability in improving global distribution. On all five test workpieces, the FCD-trained model not only achieved comprehensive improvements in DCD, EMD, and F-Score but also, surprisingly, in the CD-\(\ell_1\) metric.  We also observed a trade-off in the P2F metric, where CD and FCD had mixed results. This reveals a reasonable trade-off: FCD prioritizes uniform global coverage rather than strictly fitting the nearest ground-truth points (corresponding to the original mesh surface). Fig.~\ref{fig:abc} visualizes the results for the workpieces with the worst (ID 00010073, top row) and best (ID 00010380, bottom row) completions, showing that FCD-guided results are superior to CD in both density distribution and global structure.
\begin{figure}[t]
    \centering
    \includegraphics[width=\linewidth]{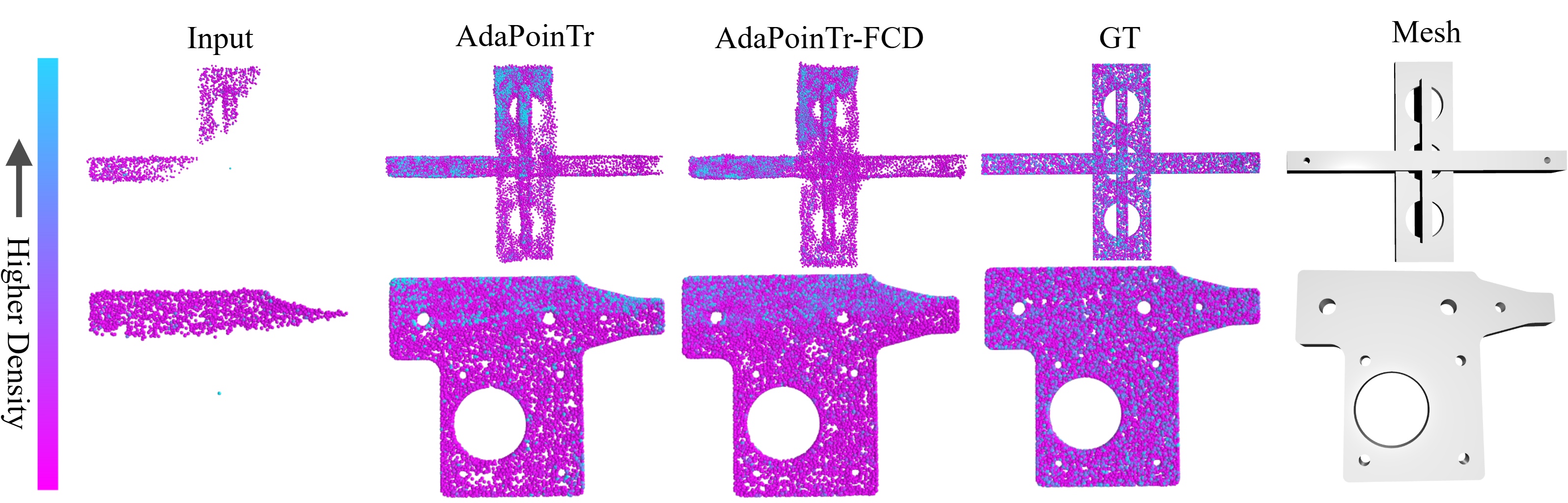}
    \caption{Visual results on the ABC dataset for industrial CAD models. The comparison highlights FCD's capability to reconstruct complex topologies with superior density distribution (bottom row) and global structural integrity (top row) compared to the standard CD baseline.}
    \label{fig:abc}
\end{figure}

\begin{table}[t]
\centering
\caption{RESULTS OF ADAPOINTR ON ABC DATASET}
\resizebox{\linewidth}{!}{%
\begin{threeparttable}
\begin{tabular}{@{}c|c|ccccc@{}}
\toprule
                                        & Workpiece ID & CD-\(\ell_1\)↓         & P2F ↓          & DCD↓           & EMD↓            & F-Score↑       \\ \midrule
\multirow{5}{*}{AdaPoinTr+CD}           & 00010045     & 6.502          & \textbf{1.257} & 0.503          & 27.946          & 0.772          \\
                                        & 00010073     & 2.798          & 0.861          & 0.417          & 18.091          & 0.975          \\
                                        & 00010221     & 5.686          & \textbf{1.035} & 0.492          & \textbf{30.838} & 0.776          \\
                                        & 00010380     & 9.903          & 8.197          & 0.590          & 23.081          & 0.809          \\
                                        & 00010760     & 4.833          & 0.959          & 0.499          & 31.676          & 0.852          \\ \midrule
\multirow{5}{*}{AdaPoinTr+FCD (Static)} & 00010045     & \textbf{5.942} & 1.340          & \textbf{0.441} & \textbf{27.385} & \textbf{0.813} \\
                                        & 00010073     & \textbf{2.523} & \textbf{0.625} & \textbf{0.403} & \textbf{18.011} & \textbf{0.987} \\
                                        & 00010221     & \textbf{5.098} & 1.095          & \textbf{0.425} & 32.303          & \textbf{0.808} \\
                                        & 00010380     & \textbf{9.754} & \textbf{6.819} & \textbf{0.567} & \textbf{20.350} & \textbf{0.824} \\
                                        & 00010760     & \textbf{4.525} & 1.393          & \textbf{0.427} & \textbf{31.371} & \textbf{0.883} \\ \bottomrule
\end{tabular}

\begin{tablenotes}
\footnotesize
\item We report the overall results under CD-\(\ell_1\) (\(\times 10^3\)), P2F (\(\times 10^3\)), EMD (\(\times 10^3\)), DCD and F-Score@1\%.
\end{tablenotes}
\end{threeparttable}
}
\label{tab:abc}
\end{table}

\subsubsection{Point Cloud Upsampling Task}
To further validate FCD's broad applicability as a general objective function, we applied it to the different generation task of point cloud upsampling, experimenting on the RepKPU \cite{rong2024repkpu} network. As shown in Table~\ref{tab:repkpu}, FCD demonstrates consistent performance gains at both 4x and 16x upsampling ratios. Compared to the standard CD baseline, the FCD-trained model achieves improvements across multiple key metrics, including CD-\(\ell_1\), HD, DCD, EMD, and F-Score. Consistent with findings on the ABC dataset (Section~\ref{para:abc}), we again observe a slight trade-off in the P2F metric. Fig.~\ref{fig:repkpu} provides visualizations for 16x upsampling. The density distributions clearly show that RepKPU guided by standard CD produces significant non-uniform clustering (blue patches) on the surfaces of the "hand" and "horse." In contrast, the FCD-guided model generates an exceptionally uniform point distribution (consistent purple color), resulting in a smoother surface that visually aligns closely with the ground truth (GT) distribution characteristics.

\begin{figure}[t]
    \centering
    \includegraphics[width=\linewidth]{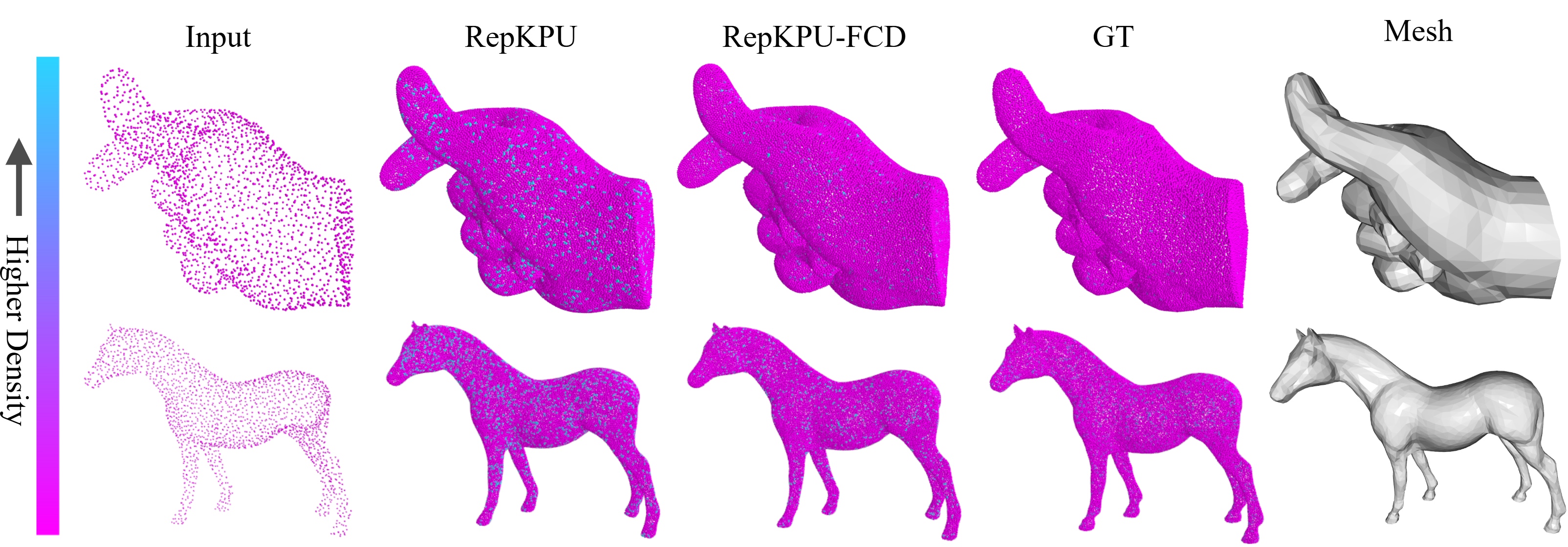}
    \caption{Visual comparison of point cloud upsampling (16x) on the PU-GAN dataset using RepKPU. The density maps reveal that the baseline method produces significant non-uniform clustering (blue patches), while RepKPU trained with FCD generates exceptionally uniform point distributions (consistent purple color) that align closely with the ground truth (GT).}
    \label{fig:repkpu}
\end{figure}

\begin{table}[t]
\centering
\caption{RESULTS OF REPKPU ON PU-GAN DATASET}
\resizebox{\linewidth}{!}{%
\begin{threeparttable}
\begin{tabular}{@{}c|c|cccccc@{}}
\toprule
                                       & Ratio & CD-\(\ell_2\)↓         & HD↓            & P2F↓           & DCD↓           & EMD↓            & F-Score↑       \\ \midrule
\multirow{2}{*}{RepKPU+CD}             & x4    & 0.246          & 1.878          & \textbf{2.467} & 0.281          & 15.909          & 0.517          \\
                                       & x16   & 0.107          & 1.936          & \textbf{2.806} & 0.244          & 47.363          & 0.887          \\ \midrule
\multirow{2}{*}{RepKPU+FCD (Static)} & x4    & \textbf{0.241} & \textbf{1.742} & 2.522          & \textbf{0.276} & \textbf{15.428} & \textbf{0.533} \\
                                       & x16   & \textbf{0.101} & \textbf{1.852} & 2.897          & \textbf{0.240} & \textbf{46.591} & \textbf{0.902} \\ \bottomrule
\end{tabular}

\begin{tablenotes}
\footnotesize
\item We report the overall results under CD-\(\ell_2\) (\(\times 10^3\)), HD (\(\times 10^3\)), P2F (\(\times 10^3\)), EMD (\(\times 10^3\)), DCD and F-Score@1\%.
\end{tablenotes}
\end{threeparttable}
}
\label{tab:repkpu}
\end{table}

\subsection{Failure Case Analysis}
\label{para:failure}
\begin{figure}[t]
    \centering
    \includegraphics[width=\linewidth]{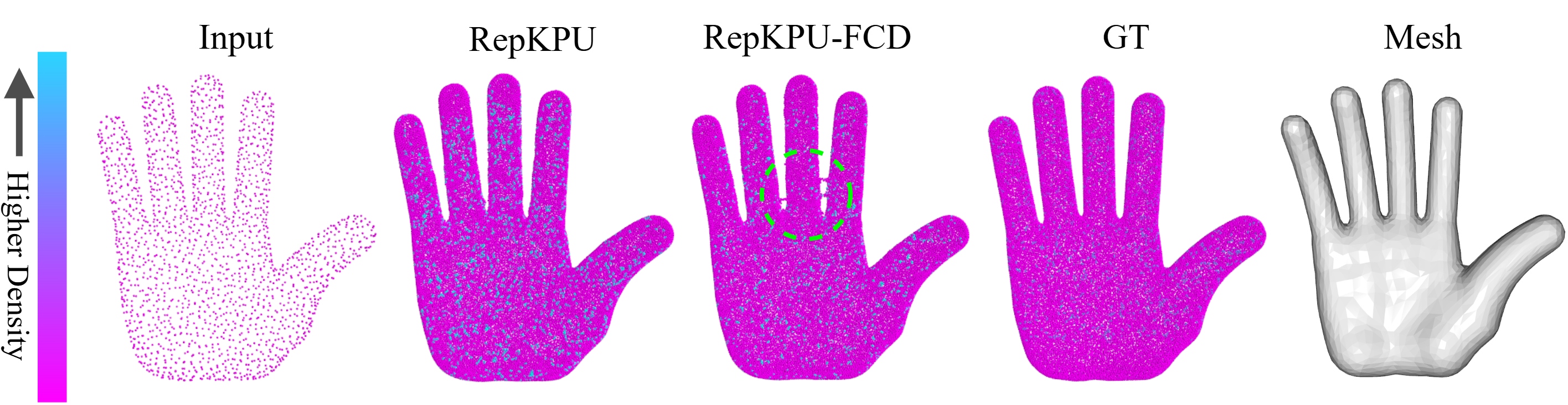}
    \caption{Visual results of a failure case in the upsampling task. While FCD improves global uniformity, its emphasis on global coverage may occasionally lead to local distortions or over-dispersion in regions requiring intricate detail preservation, as compared to the standard CD.}
    \label{fig:failure}
\end{figure}

\begin{table}[t]
\centering
\caption{COMPLEXITY ANALYSIS}
\resizebox{\linewidth}{!}{%
\begin{threeparttable}
\begin{tabular}{@{}l|llll@{}}
\toprule
\textbf{}             & Memory & Throughput      & Training Time / Epoch & Relative \\ \midrule
AdaPoinTr+CD          & 10.06G          & \textbf{65.88 samples/s} & \textbf{439.60 second}         & -                               \\
AdaPoinTr+FCD (Static) & 10.06G          & 65.31 samples/s          & 443.43 second                  & +0.87\%                         \\
AdaPoinTr+FCD (Uncertainty)    & 10.06G          & 64.63 samples/s          & 448.07 second                  & +1.93\%                         \\ \bottomrule
\end{tabular}
\begin{tablenotes}
\footnotesize
\item We report training time, GPU memory consumption, throughput and relative time increase on the AdaPoinTr network (PCN dataset, 3090 GPU).
\end{tablenotes}
\end{threeparttable}
}
\label{tab:complexity}
\end{table}

Fig.~\ref{fig:failure} visualizes an example from the point cloud upsampling task. Although the quantitative results and density distribution are superior (as shown in Table~\ref{tab:repkpu}), FCD guides the assignment switching via a relatively higher-weighted global coverage term. Consequently, compared to standard CD, FCD is more prone to introducing local distortions when applied to objects with intricate local details.

\subsection{Complexity Analysis}
We quantified the computational overhead of FCD on the AdaPoinTr network to validate its efficiency. The analysis was based on training on the PCN dataset using a single NVIDIA 3090 GPU. We selected the computationally simplest variant, FCD (Static), and the most complex variant, FCD (Uncertainty). These were compared against the baseline model (using standard CD loss) in terms of training time, GPU memory consumption, and throughput, with results presented in Table~\ref{tab:complexity}. As shown in the table, no FCD variant introduced additional memory overhead. Notably, even the most expensive configuration, FCD (Uncertainty), resulted in only a negligible 1.93\% increase in training time, with an insignificant impact on system throughput. This clearly demonstrates that the performance enhancements afforded by FCD are achieved at a minimal computational cost. For brevity, the results for other variants, whose costs fall between these two, are omitted.

\section{Conclusion and Discussion}
\label{conclusion}

We proposed the Flexible-weighted Chamfer Distance (FCD), a novel design principle and computationally negligible objective function for point cloud generation. FCD's core idea is to decouple the standard CD into local precision (\(d_{\text{CD}_{\text{local}}}\)) and global completeness (\(d_{\text{CD}_{\text{global}}}\)) sub-objectives. By adopting an asymmetric weighting strategy (\(\beta > \alpha\)), FCD prioritizes building a complete global structure, which mitigates the gradient conflicts and local minima issues (e.g., point clustering) inherent in standard CD. Extensive experiments demonstrated that this \(\beta > \alpha\) principle consistently yields significant improvements in global distribution metrics (e.g., DCD and EMD) compared to the baseline. This focus on global structure introduces a manageable trade-off with local precision metrics, and our systematic investigation into various weighting strategies (e.g., Static, Linear, Uncertainty) revealed how different schemes can be employed to navigate this balance. FCD's effectiveness and generalization were further validated across diverse tasks and data domains, including real-world scans (KITTI), complex industrial components (ABC), and point cloud upsampling (PU-GAN).

\textbf{Limitations and Future Work.}
Despite the significant efficacy of FCD, its applicability possesses inherent boundaries, primarily stemming from its design trade-offs. First, the effectiveness of FCD is contingent upon the weighting strategy. This work does not posit a single, universally optimal weighting method for all tasks; instead, we provide several generally effective schemes. The trade-off remains: an improper schedule (e.g., an excessively high \(\beta\) value) may lead to over-dispersion at the expense of local details, while degenerating to \(\alpha=\beta\) reintroduces the optimization bottlenecks of standard CD. The optimal \((\alpha, \beta)\) equilibrium point likely varies with specific task and dataset characteristics, necessitating task-specific selection or fine-tuning in practice. Second, as demonstrated in our experiments (Section~\ref{para:failure}), FCD guides assignment switching primarily via the global coverage term. Consequently, it is more prone to introducing local distortions on objects with intricate local details. Finally, the validation herein is primarily focused on supervised point cloud completion and upsampling. Future work should explore the efficacy of FCD in self-supervised, unsupervised, or large-scale scene generation settings. Furthermore, conducting more theoretical and experimental studies to establish a more adaptive or universally applicable weighting strategy remains a promising avenue for research.

\section*{Acknowledgments}
This work was supported by the Tianshan Talent Training Program (No. 2023TSYCLJ0023). We acknowledge the use of ChatGPT, Gemini, and DeepSeek for translation and language polishing of the manuscript. The scientific content and conclusions are entirely our own work.
% This should be a simple paragraph before the References to thank those individuals and institutions who have supported your work on this article.


% \appendices

% \bibliographystyle{IEEEtran}

% \end{thebibliography}

% \newpage

\bibliographystyle{IEEEtran}
\bibliography{fcd}

@inproceedings{kendall2018multi,
  title={Multi-task learning using uncertainty to weigh losses for scene geometry and semantics},
  author={Kendall, Alex and Gal, Yarin and Cipolla, Roberto},
  booktitle={Proceedings of the IEEE conference on computer vision and pattern recognition},
  pages={7482--7491},
  year={2018}
}

@article{loshchilov2018fixing,
  title={Fixing weight decay regularization in adam},
  author={Loshchilov, Ilya and Hutter, Frank},
  year={2018}
}

@article{kingma2014adam,
  title={Adam: A method for stochastic optimization},
  author={Kingma, Diederik P and Ba, Jimmy},
  journal={arXiv preprint arXiv:1412.6980},
  year={2014}
}

@article{ning2021uncertainty,
  title={Uncertainty-driven loss for single image super-resolution},
  author={Ning, Qian and Dong, Weisheng and Li, Xin and Wu, Jinjian and Shi, Guangming},
  journal={Advances in Neural Information Processing Systems},
  volume={34},
  pages={16398--16409},
  year={2021}
}

@inproceedings{qu2019attentive,
  title={Attentive history selection for conversational question answering},
  author={Qu, Chen and Yang, Liu and Qiu, Minghui and Zhang, Yongfeng and Chen, Cen and Croft, W Bruce and Iyyer, Mohit},
  booktitle={Proceedings of the 28th ACM International Conference on Information and Knowledge Management},
  pages={1391--1400},
  year={2019}
}

@inproceedings{liu2019loss,
  title={Loss-balanced task weighting to reduce negative transfer in multi-task learning},
  author={Liu, Shengchao and Liang, Yingyu and Gitter, Anthony},
  booktitle={Proceedings of the AAAI conference on artificial intelligence},
  volume={33(01)},
  pages={9977--9978},
  year={2019}
}

@inproceedings{chen2018gradnorm,
  title={Gradnorm: Gradient normalization for adaptive loss balancing in deep multitask networks},
  author={Chen, Zhao and Badrinarayanan, Vijay and Lee, Chen-Yu and Rabinovich, Andrew},
  booktitle={International conference on machine learning},
  pages={794--803},
  year={2018},
  organization={PMLR}
}

@inproceedings{belharbi2016deep,
  title={Deep multi-task learning with evolving weights.},
  author={Belharbi, Soufiane and H{\'e}rault, Romain and Chatelain, Cl{\'e}ment and Adam, S{\'e}bastien},
  booktitle={ESANN},
  year={2016}
}

@article{yeh2019flowdelta,
  title={FlowDelta: modeling flow information gain in reasoning for conversational machine comprehension},
  author={Yeh, Yi-Ting and Chen, Yun-Nung},
  journal={arXiv preprint arXiv:1908.05117},
  year={2019}
}

@inproceedings{leang2020dynamic,
  title={Dynamic task weighting methods for multi-task networks in autonomous driving systems},
  author={Leang, Isabelle and Sistu, Ganesh and B{\"u}rger, Fabian and Bursuc, Andrei and Yogamani, Senthil},
  booktitle={2020 IEEE 23rd International Conference on Intelligent Transportation Systems (ITSC)},
  pages={1--8},
  year={2020},
  organization={IEEE}
}

@inproceedings{wu2015deep,
  title={Deep neural networks employing multi-task learning and stacked bottleneck features for speech synthesis},
  author={Wu, Zhizheng and Valentini-Botinhao, Cassia and Watts, Oliver and King, Simon},
  booktitle={2015 IEEE international conference on acoustics, speech and signal processing (ICASSP)},
  pages={4460--4464},
  year={2015},
  organization={IEEE}
}

@inproceedings{hu2015fusion,
  title={Fusion of multiple parameterisations for DNN-based sinusoidal speech synthesis with multi-task learning.},
  author={Hu, Qiong and Wu, Zhizheng and Richmond, Korin and Yamagishi, Junichi and Stylianou, Yannis and Maia, Ranniery},
  booktitle={INTERSPEECH},
  pages={854--858},
  year={2015}
}

@article{liu2019multi,
  title={Multi-task deep neural networks for natural language understanding},
  author={Liu, Xiaodong and He, Pengcheng and Chen, Weizhu and Gao, Jianfeng},
  journal={arXiv preprint arXiv:1901.11504},
  year={2019}
}

@inproceedings{liu2015representation,
  title={Representation learning using multi-task deep neural networks for semantic classification and information retrieval},
  author={Liu, Xiaodong and Gao, Jianfeng and He, Xiaodong and Deng, Li and Duh, Kevin and Wang, Ye-Yi},
  booktitle={Proceedings of the 2015 Conference of the North American Chapter of the Association for Computational Linguistics: Human Language Technologies},
  pages={912--921},  
  year={2015}
}

@inproceedings{li2019pu,
  title={Pu-gan: a point cloud upsampling adversarial network},
  author={Li, Ruihui and Li, Xianzhi and Fu, Chi-Wing and Cohen-Or, Daniel and Heng, Pheng-Ann},
  booktitle={Proceedings of the IEEE/CVF international conference on computer vision},
  pages={7203--7212},
  year={2019}
}

@inproceedings{huang2020pf,
  title={Pf-net: Point fractal network for 3d point cloud completion},
  author={Huang, Zitian and Yu, Yikuan and Xu, Jiawen and Ni, Feng and Le, Xinyi},
  booktitle={Proceedings of the IEEE/CVF conference on computer vision and pattern recognition},
  pages={7662--7670},
  year={2020}
}

@inproceedings{stutz2018learning,
  title={Learning 3d shape completion from laser scan data with weak supervision},
  author={Stutz, David and Geiger, Andreas},
  booktitle={Proceedings of the IEEE conference on computer vision and pattern recognition},
  pages={1955--1964},
  year={2018}
}

@inproceedings{han2017high,
  title={High-resolution shape completion using deep neural networks for global structure and local geometry inference},
  author={Han, Xiaoguang and Li, Zhen and Huang, Haibin and Kalogerakis, Evangelos and Yu, Yizhou},
  booktitle={Proceedings of the IEEE international conference on computer vision},
  pages={85--93},
  year={2017}
}

@article{liu2019point,
  title={Point-voxel cnn for efficient 3d deep learning},
  author={Liu, Zhijian and Tang, Haotian and Lin, Yujun and Han, Song},
  journal={Advances in neural information processing systems},
  volume={32},
  year={2019}
}

@inproceedings{dai2017shape,
  title={Shape completion using 3d-encoder-predictor cnns and shape synthesis},
  author={Dai, Angela and Ruizhongtai Qi, Charles and Nie{\ss}ner, Matthias},
  booktitle={Proceedings of the IEEE conference on computer vision and pattern recognition},
  pages={5868--5877},
  year={2017}
}

@article{yang2017foldingnet,
  title={Foldingnet: Interpretable unsupervised learning on 3d point clouds},
  author={Yang, Yaoqing and Feng, Chen and Shen, Yiru and Tian, Dong},
  journal={arXiv preprint arXiv:1712.07262},
  volume={2},
  number={3},
  pages={5},
  year={2017}
}

@inproceedings{liu2020morphing,
  title={Morphing and sampling network for dense point cloud completion},
  author={Liu, Minghua and Sheng, Lu and Yang, Sheng and Shao, Jing and Hu, Shi-Min},
  booktitle={Proceedings of the AAAI conference on artificial intelligence},
  volume={34(07)},
  pages={11596--11603},
  year={2020}
}

@inproceedings{fan2017point,
  title={A point set generation network for 3d object reconstruction from a single image},
  author={Fan, Haoqiang and Su, Hao and Guibas, Leonidas J},
  booktitle={Proceedings of the IEEE conference on computer vision and pattern recognition},
  pages={605--613},
  year={2017}
}

@article{chang2015shapenet,
  title={Shapenet: An information-rich 3d model repository},
  author={Chang, Angel X and Funkhouser, Thomas and Guibas, Leonidas and Hanrahan, Pat and Huang, Qixing and Li, Zimo and Savarese, Silvio and Savva, Manolis and Song, Shuran and Su, Hao and others},
  journal={arXiv preprint arXiv:1512.03012},
  year={2015}
}

@inproceedings{yu2021pointr,
  title={Pointr: Diverse point cloud completion with geometry-aware transformers},
  author={Yu, Xumin and Rao, Yongming and Wang, Ziyi and Liu, Zuyan and Lu, Jiwen and Zhou, Jie},
  booktitle={Proceedings of the IEEE/CVF international conference on computer vision},
  pages={12498--12507},
  year={2021}
}

@article{yu2023adapointr,
  title={Adapointr: Diverse point cloud completion with adaptive geometry-aware transformers},
  author={Yu, Xumin and Rao, Yongming and Wang, Ziyi and Lu, Jiwen and Zhou, Jie},
  journal={IEEE Transactions on Pattern Analysis and Machine Intelligence},
  year={2023},
  publisher={IEEE}
}

@article{wu2021density,
  title={Density-aware chamfer distance as a comprehensive metric for point cloud completion},
  author={Wu, Tong and Pan, Liang and Zhang, Junzhe and Wang, Tai and Liu, Ziwei and Lin, Dahua},
  journal={arXiv preprint arXiv:2111.12702},
  year={2021}
}

@inproceedings{zhou2022seedformer,
  title={Seedformer: Patch seeds based point cloud completion with upsample transformer},
  author={Zhou, Haoran and Cao, Yun and Chu, Wenqing and Zhu, Junwei and Lu, Tong and Tai, Ying and Wang, Chengjie},
  booktitle={European conference on computer vision},
  pages={416--432},
  year={2022},
  organization={Springer}
}

@inproceedings{tchapmi2019topnet,
  title={Topnet: Structural point cloud decoder},
  author={Tchapmi, Lyne P and Kosaraju, Vineet and Rezatofighi, Hamid and Reid, Ian and Savarese, Silvio},
  booktitle={Proceedings of the IEEE/CVF conference on computer vision and pattern recognition},
  pages={383--392},
  year={2019}
}

@article{wang2021cascaded,
  title={Cascaded refinement network for point cloud completion with self-supervision},
  author={Wang, Xiaogang and Ang, Marcelo H and Lee, Gim Hee},
  journal={IEEE Transactions on Pattern Analysis and Machine Intelligence},
  volume={44},
  number={11},
  pages={8139--8150},
  year={2021},
  publisher={IEEE}
}

@inproceedings{wen2020point,
  title={Point cloud completion by skip-attention network with hierarchical folding},
  author={Wen, Xin and Li, Tianyang and Han, Zhizhong and Liu, Yu-Shen},
  booktitle={Proceedings of the IEEE/CVF conference on computer vision and pattern recognition},
  pages={1939--1948},
  year={2020}
}

@inproceedings{xiang2021snowflakenet,
  title={Snowflakenet: Point cloud completion by snowflake point deconvolution with skip-transformer},
  author={Xiang, Peng and Wen, Xin and Liu, Yu-Shen and Cao, Yan-Pei and Wan, Pengfei and Zheng, Wen and Han, Zhizhong},
  booktitle={Proceedings of the IEEE/CVF international conference on computer vision},
  pages={5499--5509},
  year={2021}
}

@inproceedings{tang2022lake,
  title={Lake-net: Topology-aware point cloud completion by localizing aligned keypoints},
  author={Tang, Junshu and Gong, Zhijun and Yi, Ran and Xie, Yuan and Ma, Lizhuang},
  booktitle={Proceedings of the IEEE/CVF conference on computer vision and pattern recognition},
  pages={1726--1735},
  year={2022}
}

@article{guo2020deep,
  title={Deep learning for 3d point clouds: A survey},
  author={Guo, Yulan and Wang, Hanyun and Hu, Qingyong and Liu, Hao and Liu, Li and Bennamoun, Mohammed},
  journal={IEEE transactions on pattern analysis and machine intelligence},
  volume={43},
  number={12},
  pages={4338--4364},
  year={2020},
  publisher={IEEE}
}

@inproceedings{yuan2018pcn,
  title={Pcn: Point completion network},
  author={Yuan, Wentao and Khot, Tejas and Held, David and Mertz, Christoph and Hebert, Martial},
  booktitle={2018 international conference on 3D vision (3DV)},
  pages={728--737},
  year={2018},
  organization={IEEE}
}

@inproceedings{xie2020grnet,
  title={Grnet: Gridding residual network for dense point cloud completion},
  author={Xie, Haozhe and Yao, Hongxun and Zhou, Shangchen and Mao, Jiageng and Zhang, Shengping and Sun, Wenxiu},
  booktitle={European Conference on Computer Vision},
  pages={365--381},
  year={2020},
  organization={Springer}
}

@inproceedings{qi2017pointnet,
  title={Pointnet: Deep learning on point sets for 3d classification and segmentation},
  author={Qi, Charles R and Su, Hao and Mo, Kaichun and Guibas, Leonidas J},
  booktitle={Proceedings of the IEEE conference on computer vision and pattern recognition},
  pages={652--660},
  year={2017}
}

@article{qi2017pointnet++,
  title={Pointnet++: Deep hierarchical feature learning on point sets in a metric space},
  author={Qi, Charles Ruizhongtai and Yi, Li and Su, Hao and Guibas, Leonidas J},
  journal={Advances in neural information processing systems},
  volume={30},
  year={2017}
}

@inproceedings{maturana2015voxnet,
  title={Voxnet: A 3d convolutional neural network for real-time object recognition},
  author={Maturana, Daniel and Scherer, Sebastian},
  booktitle={2015 IEEE/RSJ international conference on intelligent robots and systems (IROS)},
  pages={922--928},
  year={2015},
  organization={IEEE}
}

@inproceedings{su2015mvcnn,
  title={Multi-view convolutional neural networks for 3d shape recognition},
  author={Su, Hang and Maji, Subhransu and Kalogerakis, Evangelos and Learned-Miller, Erik},
  booktitle={Proceedings of the IEEE international conference on computer vision},
  pages={945--953},
  year={2015}
}

@article{liu2019pvcnn,
  title={Point-voxel cnn for efficient 3d deep learning},
  author={Liu, Zhijian and Tang, Haotian and Lin, Yujun and Han, Song},
  journal={Advances in neural information processing systems},
  volume={32},
  year={2019}
}

@article{li2018pointcnn,
  title={Pointcnn: Convolution on x-transformed points},
  author={Li, Yangyan and Bu, Rui and Sun, Mingchao and Wu, Wei and Di, Xinhan and Chen, Baoquan},
  journal={Advances in neural information processing systems},
  volume={31},
  year={2018}
}

@inproceedings{wu2019pointconv,
  title={Pointconv: Deep convolutional networks on 3d point clouds},
  author={Wu, Wenxuan and Qi, Zhongang and Fuxin, Li},
  booktitle={Proceedings of the IEEE/CVF Conference on computer vision and pattern recognition},
  pages={9621--9630},
  year={2019}
}

@inproceedings{zhao2019pointweb,
  title={PointWeb: Enhancing Local Neighborhood Features for Point Cloud Processing},
  author={Zhao, Hengshuang and Jiang, Li and Fu, Chi-Wing and Jia, Jiaya and Torr, Philip HS},
  booktitle={Proceedings of the IEEE/CVF Conference on Computer Vision and Pattern Recognition},
  pages={5565--5573},
  year={2019}
}

@inproceedings{thomas2019kpconv,
  title={KPConv: Flexible and Deformable Convolution for Point Clouds},
  author={Thomas, Hugues and Qi, Charles R and Deschaud, Jean-Emmanuel and Marcotegui, Beatriz and Goulette, Fran{\c{c}}ois and Guibas, Leonidas J},
  booktitle={Proceedings of the IEEE/CVF International Conference on Computer Vision},
  pages={6411--6420},
  year={2019}
}

@article{guo2021pct,
  title={Pct: Point cloud transformer},
  author={Guo, Meng-Hao and Cai, Jun-Xiong and Liu, Zheng-Ning and Mu, Tai-Jiang and Martin, Ralph R and Hu, Shi-Min},
  journal={Computational Visual Media},
  volume={7},
  pages={187--199},
  year={2021},
  publisher={Springer}
}

@inproceedings{zhao2021pointtransformer,
  title={Point transformer},
  author={Zhao, Hengshuang and Jiang, Li and Jia, Jiaya and Torr, Philip HS and Koltun, Vladlen},
  booktitle={Proceedings of the IEEE/CVF international conference on computer vision},
  pages={16259--16268},
  year={2021}
}

@inproceedings{yu2022pointbert,
  title={Point-bert: Pre-training 3d point cloud transformers with masked point modeling},
  author={Yu, Xumin and Tang, Lulu and Rao, Yongming and Huang, Tiejun and Zhou, Jie and Lu, Jiwen},
  booktitle={Proceedings of the IEEE/CVF conference on computer vision and pattern recognition},
  pages={19313--19322},
  year={2022}
}

@article{wang2019dgcnn,
  title={Dynamic graph cnn for learning on point clouds},
  author={Wang, Yue and Sun, Yongbin and Liu, Ziwei and Sarma, Sanjay E and Bronstein, Michael M and Solomon, Justin M},
  journal={ACM Transactions on Graphics (tog)},
  volume={38},
  number={5},
  pages={1--12},
  year={2019},
  publisher={Acm New York, NY, USA}
}

@inproceedings{wen2021pmp,
  title={Pmp-net: Point cloud completion by learning multi-step point moving paths},
  author={Wen, Xin and Xiang, Peng and Han, Zhizhong and Cao, Yan-Pei and Wan, Pengfei and Zheng, Wen and Liu, Yu-Shen},
  booktitle={Proceedings of the IEEE/CVF conference on computer vision and pattern recognition},
  pages={7443--7452},
  year={2021}
}

@inproceedings{yu2024geoformer,
  title={Geoformer: Learning point cloud completion with tri-plane integrated transformer},
  author={Yu, Jinpeng and Huang, Binbin and Zhang, Yuxuan and Li, Huaxia and Tang, Xu and Gao, Shenghua},
  booktitle={Proceedings of the 32nd ACM International Conference on Multimedia},
  pages={8952--8961},
  year={2024}
}

@inproceedings{rong2024crapcn,
  title={Cra-pcn: Point cloud completion with intra-and inter-level cross-resolution transformers},
  author={Rong, Yi and Zhou, Haoran and Yuan, Lixin and Mei, Cheng and Wang, Jiahao and Lu, Tong},
  booktitle={Proceedings of the AAAI Conference on Artificial Intelligence},
  volume={38},
  number={5},
  pages={4676--4685},
  year={2024}
}

@inproceedings{wei2025pcdreamer,
  title={Pcdreamer: Point cloud completion through multi-view diffusion priors},
  author={Wei, Guangshun and Feng, Yuan and Ma, Long and Wang, Chen and Zhou, Yuanfeng and Li, Changjian},
  booktitle={Proceedings of the Computer Vision and Pattern Recognition Conference},
  pages={27243--27253},
  year={2025}
}

@inproceedings{duan2024tcorresnet,
  title={T-CorresNet: template guided 3D point cloud completion with correspondence pooling query generation strategy},
  author={Duan, Fan and Yu, Jiahao and Chen, Li},
  booktitle={European Conference on Computer Vision},
  pages={90--106},
  year={2024},
  organization={Springer}
}

@article{lin2023infocd,
  title={InfoCD: A contrastive chamfer distance loss for point cloud completion},
  author={Lin, Fangzhou and Yue, Yun and Zhang, Ziming and Hou, Songlin and Yamada, Kazunori and Kolachalama, Vijaya and Saligrama, Venkatesh},
  journal={Advances in Neural Information Processing Systems},
  volume={36},
  pages={76960--76973},
  year={2023}
}

@inproceedings{lin2023hyperbolic,
  title={Hyperbolic chamfer distance for point cloud completion},
  author={Lin, Fangzhou and Yue, Yun and Hou, Songlin and Yu, Xuechu and Xu, Yajun and Yamada, Kazunori D and Zhang, Ziming},
  booktitle={Proceedings of the IEEE/CVF international conference on computer vision},
  pages={14595--14606},
  year={2023}
}

@article{geiger2013kitti,
  title={Vision meets robotics: The kitti dataset},
  author={Geiger, Andreas and Lenz, Philip and Stiller, Christoph and Urtasun, Raquel},
  journal={The international journal of robotics research},
  volume={32},
  number={11},
  pages={1231--1237},
  year={2013},
  publisher={Sage Publications Sage UK: London, England}
}

@inproceedings{koch2019abc,
  title={Abc: A big cad model dataset for geometric deep learning},
  author={Koch, Sebastian and Matveev, Albert and Jiang, Zhongshi and Williams, Francis and Artemov, Alexey and Burnaev, Evgeny and Alexa, Marc and Zorin, Denis and Panozzo, Daniele},
  booktitle={Proceedings of the IEEE/CVF conference on computer vision and pattern recognition},
  pages={9601--9611},
  year={2019}
}

@inproceedings{rong2024repkpu,
  title={Repkpu: Point cloud upsampling with kernel point representation and deformation},
  author={Rong, Yi and Zhou, Haoran and Xia, Kang and Mei, Cheng and Wang, Jiahao and Lu, Tong},
  booktitle={Proceedings of the IEEE/CVF Conference on Computer Vision and Pattern Recognition},
  pages={21050--21060},
  year={2024}
}

@article{berger2013benchmark,
  title={A benchmark for surface reconstruction},
  author={Berger, Matthew and Levine, Joshua A and Nonato, Luis Gustavo and Taubin, Gabriel and Silva, Claudio T},
  journal={ACM Transactions on Graphics (TOG)},
  volume={32},
  number={2},
  pages={1--17},
  year={2013},
  publisher={ACM New York, NY, USA}
}

\begin{IEEEbiography}[{\includegraphics[width=1in,height=1.25in,clip,keepaspectratio]{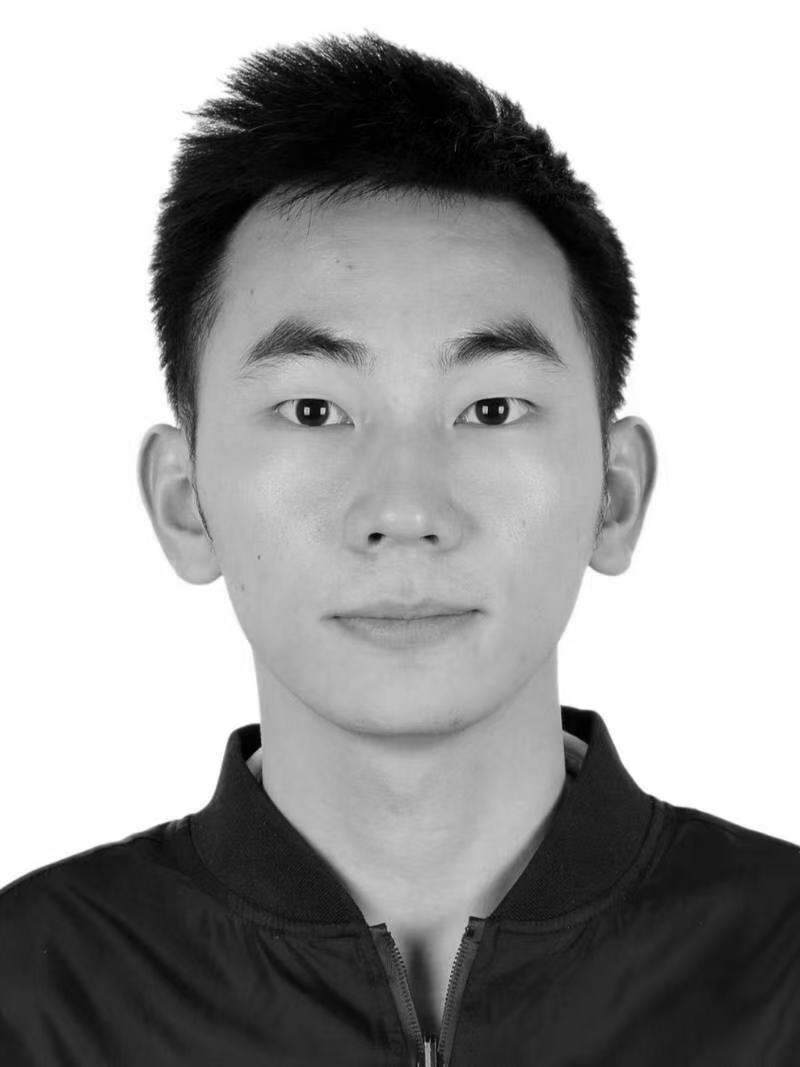}}]{Jie Li} received the B.S. degree from Hefei University of Technology, Hefei, China, in 2019. He is currently pursuing the Ph.D. degree at the College of Computer Science and Technology of Xinjiang University. His research interests include pattern recognition and 3D vision.
\end{IEEEbiography}

\begin{IEEEbiography}[{\includegraphics[width=1in,height=1.25in,clip,keepaspectratio]{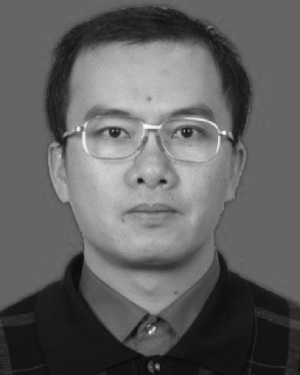}}]{Shengwei Tian} received the BS, MS, and PhD degrees from the School of Information Science and Engineering, Xinjiang University, Urumqi, China, in 1997, 2004, and 2010, respectively. Since 2002, he has been a teacher with the School of Software, Xinjiang University, where he is currently a professor. His research interests include artificial intelligence, natural language processing, and cyberspace security.
\end{IEEEbiography}

\begin{IEEEbiography}[{\includegraphics[width=1in,height=1.25in,clip,keepaspectratio]{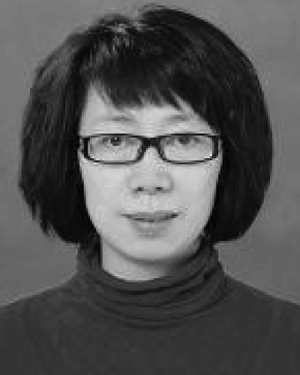}}]{Long Yu} received the B.S. and M.S. degrees from the College of Information Science and Engineering, Xinjiang University, Urumqi, China, in 1997 and 2008, respectively. Since 2002, she has been a Teacher with the College of Information Science and Engineering, Xinjiang University, where she is currently a Professor. Her research interests include artificial intelligence, data mining, natural language processing, and cyberspace security.
\end{IEEEbiography}

\begin{IEEEbiography}[{\includegraphics[width=1in,height=1.25in,clip,keepaspectratio]{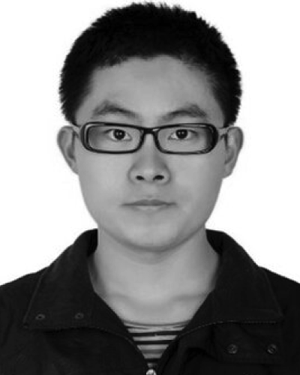}}]{Xin Ning} received his Ph.D. in 2017 from Institute of Semiconductors, Chinese Academy of Sciences. He is currently an Assistant Professor of Artificial Intelligence at Institute of Semiconductors Chinese Academy of Sciences. His research interests include deep learning machine art, pattern recognition, and image cognitive computation. He is a member of IEEE.
\end{IEEEbiography}

\vfill

\end{document}